\documentclass[letterpaper]{article} 
\usepackage{aaai2026}  
\usepackage{times}  
\usepackage{helvet}  
\usepackage{courier}  
\usepackage[hyphens]{url}  
\usepackage{graphicx} 
\usepackage{natbib}  
\usepackage{caption} 
\usepackage{algorithm}
\usepackage{algorithmic}

\usepackage{newfloat}
\usepackage{listings}
\DeclareCaptionStyle{ruled}{labelfont=normalfont,labelsep=colon,strut=off} 
\floatstyle{ruled}
\newfloat{listing}{tb}{lst}{}
\floatname{listing}{Listing}
\usepackage{amsmath, amssymb, amsthm, amsfonts}
\usepackage{mathtools} 
\usepackage{bm}        
\usepackage{subcaption}

\newtheorem{theorem}{Theorem}
\newtheorem{definition}{Definition}
\newtheorem{assumption}{Assumption}
\newtheorem{proposition}{Proposition}

\newcommand{\calG}{\mathcal{G}}
\newcommand{\calN}{\mathcal{N}}
\newcommand{\calM}{\mathcal{M}}
\newcommand{\calP}{\mathcal{P}}

\newcommand{\calX}{\mathcal{X}}

\newcommand{\E}{\mathbb{E}}

\newcommand{\bmS}{\bm{S}}
\newcommand{\bmQ}{\bm{Q}}
\newcommand{\bmX}{\bm{X}}

\DeclareMathOperator{\KL}{KL}
\DeclareMathOperator{\doop}{do}

\usepackage{algorithm}
\usepackage{algorithmic}

\usepackage{booktabs}   
\usepackage{multirow}   
\usepackage{makecell}   

\title{Spatio-Temporal Hierarchical Causal Models}
\author{
    Xintong Li\textsuperscript{\rm 1},
    Haoran Zhang\textsuperscript{\rm 1},
    Xiao Zhou\textsuperscript{\rm 2,3,4}\thanks{Corresponding author.}
}
\affiliations{
    \textsuperscript{\rm 1}School of Statistics, Renmin University of China Beijing, China\\
    \textsuperscript{\rm 2}Gaoling School of Artificial Intelligence, Renmin University of China Beijing, China\\
    \textsuperscript{\rm 3}Beijing Key Laboratory of Research on Large Models and Intelligent Governance\\
    \textsuperscript{\rm 4}Engineering Research Center of Next-Generation Intelligent Search and Recommendation, MOE

    \{lixintong, zhanghaoran1015, xiaozhou\}@ruc.edu.cn
}

\usepackage{bibentry}

\makeatletter
\renewcommand\appendix{
  \par
  \setcounter{section}{0}%
  \setcounter{subsection}{0}%
  \gdef\thesection{\Alph{section}}%
  \gdef\thesubsection{\thesection.\arabic{subsection}}%
  \@addtoreset{subsection}{section}%
}
\makeatother


\begin{document}

\maketitle

\begin{abstract}
The abundance of fine-grained spatio-temporal data, such as traffic sensor networks, offers vast opportunities for scientific discovery. However, inferring causal relationships from such observational data remains challenging, particularly due to unobserved confounders that are specific to units (e.g., geographical locations) yet influence outcomes over time. Most existing methods for spatio-temporal causal inference assume that all confounders are observed, an assumption that is often violated in practice.
In this paper, we introduce \textit{Spatio-Temporal Hierarchical Causal Models} (\textbf{ST-HCMs}), a novel graphical framework that extends hierarchical causal modeling to the spatio-temporal domain. At the core of our approach is the \textit{Spatio-Temporal Collapse Theorem}, which shows that a complex ST-HCM converges to a simpler flat causal model as the amount of subunit data increases. This theoretical result enables a general procedure for \textit{causal identification}, allowing ST-HCMs to recover causal effects even in the presence of unobserved, time-invariant unit-level confounders, a scenario where standard non-hierarchical models fail.
We validate the effectiveness of our framework on both synthetic and real-world datasets, demonstrating its potential for robust causal inference in complex dynamic systems.
\end{abstract}

\begin{links}
    \link{Code}{https://github.com/CAMELLIAxt/ST-HCMs}
\end{links}

\section{Introduction}
Fine-grained spatio-temporal data are becoming ubiquitous, generated by sources ranging from satellite imagery in environmental science to sensor networks in urban planning ~\citep{wang_zhou_mascolo_noulas_xie_liu_2018, xu2024cgapurbanregionrepresentation, ali2024causalityearthscience,dong2024causallyawarespatiotemporalmultigraphconvolution,Li_Li_Zhou_2025}. Such data inherently possess a hierarchical structure, offering unprecedented opportunities to uncover causal mechanisms in complex dynamic systems. 
However, reliable causal inference from observational data remains challenging ~\citep{hernan2010causal,reich2021review}, obscuring mechanistic insights and impeding accurate policy intervention predictions in complex spatio-temporal systems like urban operations ~\cite{zhang2023spatiotemporal,DONG2025104244,xie-etal-2025-coalign}.

Consider the urban traffic system in Figure~\ref{fig:intro}, where a city is divided into several regions (units), and within each region, numerous sensors monitor traffic flow (outcome) on specific road segments (subunits). A traffic accident (treatment) can disrupt the flow, but a region's unique characteristics, such as its geographical layout and road network structure, serve as an \textbf{unobserved confounder}. This confounder affects both the likelihood of accidents and typical traffic patterns, making it difficult to disentangle the true causal effect of an accident from the region's inherent characteristics ~\citep{weinstein2024hierarchicalcausalmodels,10.1093/ectj/utaf011}. The problem is further compounded by temporal dynamics and spatial spillovers from neighboring regions, which introduces additional layers of complexity~\citep{ROBINS19861393,10.1214/16-AOAS1005,reich2021review}.

\begin{figure}[t]
    \centering
    \includegraphics[width=0.46\textwidth]{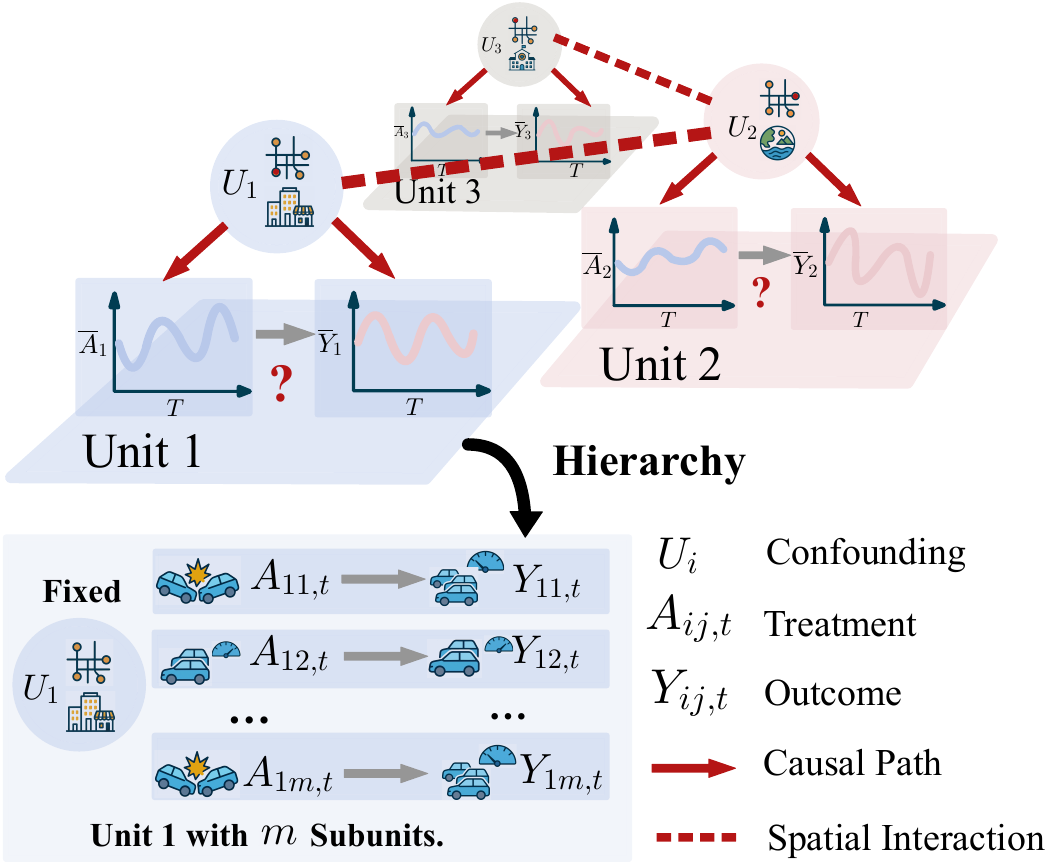}
    \caption{An example of hierarchical causality in traffic.} 
    \label{fig:intro}
\end{figure}

This hierarchical structure offers a powerful, yet often underutilized, path forward. Standard analytical methods often bypass this structure, instead relying on the strong ``no unmeasured confounding" assumption~\citep{hernan2006estimating, ali2024estimating,oprescu2025gstunetspatiotemporalcausalinference}, which is untenable for complex systems like the one in Figure~\ref{fig:intro}. While instrumental variables (IV) provide a theoretical alternative~\citep{martinussen2017instrumental,cheng2024instrumental}, finding a valid instrument that is both relevant and exogenous is notoriously difficult in practice. Methods like fixed-effect (FE) models do leverage within-unit variation~\citep{angrist2009mostly,pmlr-v115-jensen20a}, but they typically impose restrictive parametric assumptions, such as linearity, limiting their applicability.

In contrast, our approach builds on the key insight illustrated in Figure~\ref{fig:intro}: we treat each unit as a local natural experiment, enabling non-parametric estimation of region-specific causal effects by analyzing the relationship between treatments and subunit-level outcomes. This paper extends this powerful principle to spatio-temporal settings, where such effects are confounded not only by static heterogeneity but also by dynamic evolution and spatial interactions. 

While the idea of leveraging hierarchical structure for causal inference has been explored in static settings ~\citep{weinstein2024hierarchicalcausalmodels}, its extension to the spatio-temporal domain presents non-trivial. The introduction of time and space creates two fundamental challenges: (1)~\textbf{dynamic confounding}, where the history of the system acts as a time-varying confounder for future treatments and outcomes, and (2)~\textbf{spatial confounding}, where the states of neighboring units interfere with each other's causal processes. To address these challenges, we introduce \textbf{Spatio-Temporal Hierarchical Causal Models (ST-HCMs)}, a novel graphical framework that explicitly models these dependencies. The core of our solution is a \textbf{Spatio-Temporal Collapse Theorem}, which provides a theoretical guarantee that the complex, dynamic ST-HCM mathematically converges to a simpler, equivalent representation. This theorem is the key that enables causal identification, allowing us to disentangle effects from the web of spatio-temporal confounding.

The contributions of this paper are as follows:
\begin{itemize}
    \item  \textbf{Conceptually,} we propose the first graphical framework (ST-HCMs) for causal modeling of nested spatio-temporal data, providing a formal formulation to reason about causality within spatio-temporal systems.
    \item \textbf{Theoretically,} we prove a Spatio-Temporal Collapse Theorem (Theorem~\ref{thm:collapse_st}) that guarantees a complex ST-HCM converges to a simpler, equivalent model, which in turn enables causal identification (Theorem~\ref{thm:id_st_hcm_adj},\ref{thm:id_st_hcm_iv}) despite unobserved confounders.
    \item \textbf{Experimentally,} we validate our framework through extensive simulations, showing that it robustly recovers causal effects despite unobserved confounders and spatial spillovers. On a real-world dataset, we further demonstrate that it substantially reduces the biases inherent to traditional non-hierarchical models.
\end{itemize}    

\section{Preliminaries}
\label{sec:preliminaries}

\paragraph{Notation}
Consider a system composed of $N$ units, indexed by $i \in \{1, \dots, N\}$, observed over $T$ discrete time steps, indexed by $t \in \{1, \dots, T\}$. Each unit $i$ contains $m_i$ subunits, indexed by $j \in \{1, \dots, m_i\}$. For simplicity, we assume $m_i = m$ for all $i$.
Let $V$ be a finite index set for endogenous variable types. This set is partitioned into subunit-level variables $S$ and unit-level variables $U = V \setminus S$.

We define the key variables as follows:
\begin{itemize}
    \item $X_{ij,t}^v$: The value of a subunit-level variable $v \in S$ for subunit $j$ in unit $i$ at time $t$.
    \item $X_{i,t}^w$: The value of a unit-level variable $w \in U$ for unit $i$ at time $t$.
    \item $\bmX_{i,<t}$: The full history of all variables associated with unit $i$ prior to time $t$.
    \item $\calN(i)$: The set of spatial neighbors of unit $i$.
    \item $\calX$: A generic state space. $\calP(\calX)$ denotes the space of all probability measures on $\calX$.
    \item $U_i$: A time-invariant, unobserved random variable representing the static latent properties of unit $i$.
    \item $\gamma_{i,t}^v, \epsilon_{ij,t}^v$: Time-varying exogenous noise variables at the unit and subunit levels, respectively.
\end{itemize}

\paragraph{Causal Graphical Models and HCMs}
A Causal Graphical Model (CGM) represents causal relationships as a Directed Acyclic Graph (DAG), where nodes are variables and directed edges denote causal effects. Interventions are denoted by the $\doop$-operator \citep{10.1214/09-SS057}.

\begin{figure}[t]
    \centering
    \includegraphics[width=0.45\textwidth]{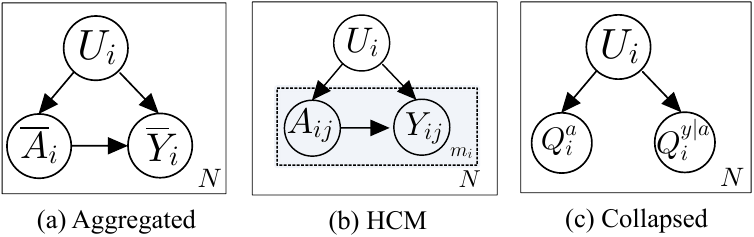}
    \caption{Hierarchy and Collapsing in HCMs.} 
    \label{fig:HCM_org_plot}
\end{figure}

The static Hierarchical Causal Model (HCM) framework \cite{weinstein2024hierarchicalcausalmodels} is designed to overcome confounding in nested data. As illustrated in Figure~\ref{fig:HCM_org_plot}(a), in a standard aggregated view, an unobserved confounder $U_i$ biases the estimated effect of treatment $\bar{A}_i$ on outcome $\bar{Y}_i$. The key insight of HCM is to leverage the fine-grained subunit data within each unit, as depicted in Figure~\ref{fig:HCM_org_plot}(b). 
This hierarchical structure enables a powerful identification strategy known as collapsing. The procedure transforms the HCM into an equivalent, flat causal graph (Figure~\ref{fig:HCM_org_plot}(c)), where the subunit-level causal mechanisms are succinctly captured by new random variables, termed Q-variables (e.g., $Q_i^a, Q_i^{y|a}$). However, the static HCM is limited to cross-sectional settings and does not account for temporal dynamics or spatial dependencies, which are essential in spatio-temporal domains. Our work extends this collapsing principle to the dynamic spatio-temporal domain.

\section{Spatio-Temporal Hierarchical Causal Models}
\label{sec:st-hcm}

In this section, we introduce our Spatio-Temporal Hierarchical Causal Model (ST-HCM) framework. To build intuition, we first present a temporal extension of the static HCM, the T-HCM. We then incorporate spatial dependencies to formulate the ST-HCM and present the convergence theorem.

\subsection{ A Temporal Extension: T-HCMs}
\label{subsec:t-hcm}

\begin{definition}[Temporal Hierarchical Structural Causal Model (T-HSCM)]
A T-HSCM, $\calM_{t-hscm}$, is defined by: (1) a summary causal graph $\calG_T$ whose unrolling over time yields a DAG; (2) a set of static exogenous variables $U_i \sim P(U)$; (3) sets of time-varying exogenous noises $\{\gamma_{i,t}^v\}, \{\epsilon_{ij,t}^v\}$; and (4) a set of deterministic mechanism functions $\{f^v\}_{v \in V}$. The value of each endogenous variable is generated as follows:
\begin{itemize}
    \item For a subunit-level variable $v \in S$:
    \begin{align*}
         x_{ij,t}^v = f^v(U_i, \text{pa}_U(x_{i,t}^v), \text{pa}_S(x_{ij,t}^v), \bmX_{i,<t}, \gamma_{i,t}^v, \epsilon_{ij,t}^v) .
    \end{align*}
    \item For a unit-level variable $w \in U$:
    \begin{align*}
        x_{i,t}^w &= f^w(U_i, \text{pa}_U(x_{i,t}^w), \\
        &\{\{x_{ij,t}^{v'}\}_{v' \in \text{pa}_S(w,t)}\}_{j=1}^m, \bmX_{i,<t}, \gamma_{i,t}^w),
    \end{align*}
    where the mechanism $f^w$ is permutation-invariant with respect to the subunit index $j$.
\end{itemize}
\end{definition}

By integrating out the exogenous noises from the structural model, we obtain the T-HCGM. This graphical model is described by stochastic mechanisms, centered on the dynamic evolution of Q-variables, which represent the conditional distributions governing subunit-level phenomena.

\begin{definition}[Temporal Hierarchical Causal Graphical Model (T-HCGM)]
A T-HCGM, $\calM_{t-hcgm}$, corresponding to a T-HSCM, is characterized by a two-stage generative process at each time step $t$:
\begin{enumerate}
    \item \textbf{Q-variable Evolution:} A vector of Q-variables $\bmQ_{i,t}$ is generated for each unit $i$, conditional on its macro-history: $\bmQ_{i,t} \sim p(\bm{q}_t | U_i, \bm{S}_{i,<t})$
    where $\bm{S}_{i,<t} = (X_{i,<t}^U, \bmQ_{i,<t})$ is the macro-history.
    \item \textbf{Variable Generation:} Endogenous variables are drawn conditional on the current Q-variables and their history.
\end{enumerate}
\end{definition}

To prove the convergence of this model to a simpler, collapsed representation, we require the following assumptions.

\begin{assumption}[Mechanism Convergence]
\label{assump:mech_conv_t}
For any $w \in U$, its conditional distribution's generating mechanism is continuous with respect to the empirical measure of its subunit-level parent histories. Let $\calP_{i,<t}$ be the true probability measure on the subunit history space, determined by the macro-history $\text{Hist}_{i,<t}^{col} = (U_i, \bm{S}_{i,<t})$. Let $\hat{\calP}_{i,m,<t}$ be the empirical measure from $m$ subunits. We assume:
\begin{align} 
\lim_{m\to\infty} \E_{\hat{\calP}_{i,m,<t}} [ \KL( p(X_{i,t}^w | \bm{S}_{i,<t}, \calP_{i,<t}) \notag \\
|| \ p(X_{i,t}^w | \bm{S}_{i,<t}, \hat{\calP}_{i,m,<t}) ) ] = 0 .
\end{align}
\end{assumption}
This assumption formalizes the idea that with enough subunit data, the uncertainty from sampling subunit histories becomes negligible for determining unit-level outcomes .

\begin{assumption}[Regularity of State Spaces]
\label{assump:regularity_t}
The state spaces for all variables are Polish spaces. Furthermore, the class of functions defining subunit histories forms a Glivenko-Cantelli class, ensuring the uniform convergence of empirical measures.
\end{assumption}
This is a standard technical condition for ensuring the regularity of the underlying probability spaces ~\cite{hernan2006estimating, ali2024estimating}.

\begin{theorem}[Convergence of T-HCMs]
\label{thm:collapse_t}
Under Assumptions \ref{assump:mech_conv_t} and \ref{assump:regularity_t}, as the number of subunits $m \to \infty$, the joint distribution of the macro-state history from a T-HCGM, marginalized over subunit variables ($P_{m,T}^{marg}$), converges in Kullback-Leibler divergence to the distribution of its corresponding Dynamic Collapsed Model ($P_{col,T}$). For any finite time $T$:
\begin{align}
     \lim_{m\to\infty} \KL(P_{col,T} \ || \ P_{m,T}^{marg}) = 0 .
\end{align}
\end{theorem}

\begin{proof}[Proof Sketch]
The proof proceeds by mathematical induction on the time step $t$, using the chain rule: $ \KL(P_t || Q_t) = \KL(P_{t-1} || Q_{t-1}) + \E_{P_{t-1}}[\KL(P(S_t|H_{t-1}) || Q(S_t|H_{t-1}))] $. The inductive step shows that the expected KL-divergence of the one-step transition kernels converges to zero. We decompose this inner KL term into two parts: one for the Q-variable evolution and one for the unit-level variable generation. The KL-divergence for the Q-variable part is identically zero by definition, as the collapsed model preserves this mechanism exactly. The KL-divergence for the unit-level variable part converges to zero due to a continuous mapping argument: Assumption \ref{assump:regularity_t} ensures the convergence of the input, and Assumption \ref{assump:mech_conv_t} ensures the continuity of the function. Thus, the output distributions converge. The full proof is provided in Appendix~\ref{app:proof_t_hcm_collapse}.
\end{proof}

\subsection{The Framework: ST-HCMs}
\label{subsec:st-hcm}

We now incorporate spatial dependencies to construct the full Spatio-Temporal Hierarchical Causal Model. This introduces new challenges, primarily the symmetric, contemporaneous interactions between units, which can create cycles in the causal graph. We address this by imposing a causal ordering on spatial interactions.

\begin{definition}[Spatio-Temporal Hierarchical Causal Graphical Model (ST-HCGM)]
An ST-HCGM, $\calM_{st-hcgm}$, extends the T-HCGM by allowing the generating mechanisms to depend on the states of spatial neighbors $\calN(i)$.
\begin{itemize}
    \item The Q-variable evolution for unit $i$ now depends on its own macro-history and that of its neighbors:
    \begin{align}
         \bmQ_{i,t} \sim p(\bm{q}_t | U_i, \bm{S}_{i,<t}, \{\bm{S}_{k,<t}\}_{k \in \calN(i)}) .
    \end{align}
    \item The generation of unit-level variables $X_{i,t}^w$ similarly depends on the history of unit $i$ and its neighbors $\calN(i)$.
\end{itemize}
\end{definition}

To ensure the model is well-defined, we introduce the following assumptions regarding its spatial structure.

\begin{assumption}[Spatial Markov Property (S1)]
\label{assump:spatial_markov}
The state of any unit $i$ at time $t$ is conditionally independent of all non-neighboring units, given the history of unit $i$ and its designated neighbors $\calN(i)$. 
\end{assumption}

\begin{assumption}[Spatial Causal Ordering (S2)]
\label{assump:spatial_order}
There exists a known, global ordering $<$ on the set of all units $\{1, \dots, N\}$. For any two neighbors $i, j \in \calN(i)$, contemporaneous causal influence is unidirectional: unit $j$ can influence unit $i$ at time $t$ only if $j < i$. 
\end{assumption}
\begin{assumption}[Spatial Homogeneity (S3)]
\label{assump:spatial_homogeneity}
All units share the same set of mechanism functions $\{f^v\}$. The functional form of spatial influence is invariant across space, depending only on the relative relationship between units, not their absolute identities .
\end{assumption}

With these assumptions in place, we can state the main convergence theorem for our full framework.

\begin{theorem}[Convergence of ST-HCMs]
\label{thm:collapse_st}
Under Assumptions \ref{assump:mech_conv_t}, \ref{assump:regularity_t}, and \ref{assump:spatial_markov}-\ref{assump:spatial_homogeneity}, as the number of subunits $m \to \infty$, the joint distribution of the macro-state history from an ST-HCGM, marginalized over subunit variables ($P_{m,T,N}^{marg}$), converges in Kullback-Leibler divergence to the distribution of its corresponding Spatio-Temporal Dynamic Collapsed Model ($P_{col,T,N}$). For any finite $N$ and $T$:
\begin{align}
    \lim_{m\to\infty} \KL(P_{col,T,N} \ || \ P_{m,T,N}^{marg}) = 0 .
\end{align}
\end{theorem}

\begin{proof}[Proof Sketch]
We construct a single ``Super-Unit'' model, $\calM_{super}$, whose state at time $t$ is the concatenated vector of states of all $N$ spatial units, i.e., $\bm{S}_{super,t} = (\bm{S}_{1,t}, \dots, \bm{S}_{N,t})$. Assumption \ref{assump:spatial_order} guarantees that the internal dependencies within this Super-Unit's state vector at any time $t$ are acyclic. Consequently, $\calM_{super}$ is a valid (though high-dimensional) T-HCM that satisfies all premises of Theorem \ref{thm:collapse_t}. Since the collapsing process for $\calM_{super}$ is mathematically equivalent to the collapsing process for the original ST-HCM, the convergence guaranteed by Theorem \ref{thm:collapse_t} for $\calM_{super}$ directly implies the convergence for the ST-HCM. The full proof is provided in Appendix~\ref{app:proof_st_hcm_collapse}.
\end{proof}

\section{Identification and Estimation}
\label{sec:inference}
Having established definition and convergence properties for ST-HCMs, we now address its primary purpose: enabling causal inference from spatio-temporal data. This section tackles identifiability, determining when causal effects can be uniquely computed from observational data. We develop sufficient conditions for identifiability in ST-HCMs.

\subsection{The Identifiability Problem in Dynamic HCMs}
\label{subsec:id_problem}
The fundamental challenge of causal inference in dynamic hierarchical systems is threefold (Figure~\ref{fig:iden_plot}). First, we face unobserved unit-level confounding, represented by variable $U_i$. Second, the system's history acts as a confounder for current treatments and future outcomes. Third, neighboring units' states can confound effects, creating spatial dependencies that traditional methods cannot address.

\begin{figure}[t]
    \centering
    \includegraphics[width=0.48\textwidth]{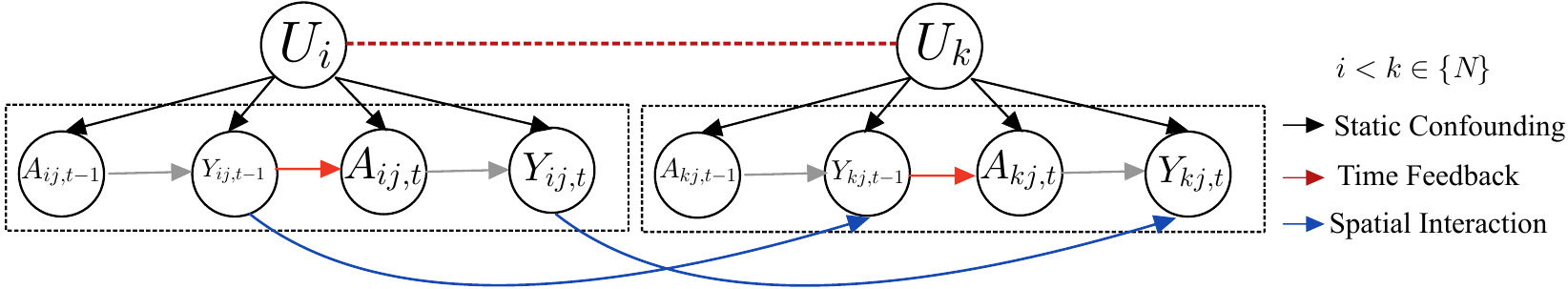}
    \caption{The challenge of identifiability.} 
    \label{fig:iden_plot}
\end{figure}

Our goal is to identify the full post-intervention distribution $P(Y_{k,t} | \doop(A_{i,t'} = a^*))$, which describes the outcome at a location $k$ and time $t$ under an intervention on a treatment at location $i$ and time $t'$. Thanks to the convergence theorems (Theorems \ref{thm:collapse_t} and \ref{thm:collapse_st}), this problem is equivalent to identifying the effect in the corresponding Dynamic Collapsed Model (DCM). Our identifiability results are derived by analyzing the graphical structure of this DCM.

\subsection{Identifiability via Sequential Adjustability}
\label{subsec:id_adjustability}

Our first identification strategy relies on the principle of adjusting for a sufficiently rich set of observed covariates to block all backdoor paths between treatment and outcome.
We begin by stating the result for the purely temporal case, which forms the foundation for the more general theorem.

\begin{proposition}[Identifiability in T-HCMs via Adjustability]
\label{prop:id_t_hcm_adj}
In a T-HCGM satisfying the premises of Theorem \ref{thm:collapse_t}, the causal effect of a subunit-level treatment $A_{ij,t'}$ on a future outcome $Y_{ij,t}$ ($t \ge t'$) is identifiable if, in its corresponding DCM, all backdoor paths from the treatment node $Q_{i,t'}^A$ to the outcome node $Q_{i,t}^Y$ are blocked by conditioning on the observable macro-history $\bm{S}_{i,<t'}$.
\end{proposition}

\begin{proof}[Proof Sketch]
The proof operates on the DCM. Conditioning on the observable history $\bm{S}_{i,<t'}$ blocks all confounding paths originating from the system's dynamic evolution. The only remaining backdoor path is the static confounding path $Q_{i,t'}^A \leftarrow U_i \rightarrow Q_{i,t}^Y$. This path is resolved by the hierarchical structure itself: we first estimate the conditional effect within each unit $i$ (where $U_i$ is fixed), and then average these estimates across all units. This two-step procedure of sequential adjustment followed by hierarchical averaging allows for identification. The full proof is in Appendix~\ref{app:prop_id_t_hcm_adj}.
\end{proof}
This principle naturally extends to the full spatio-temporal setting, where the adjustment set must also include the history of spatial neighbors. 

\begin{theorem}[Identifiability in ST-HCMs via Adjustability]
\label{thm:id_st_hcm_adj}
In an ST-HCGM satisfying the premises of Theorem \ref{thm:collapse_st}, the causal effect $P(Y_{k,t} | \doop(A_{i,t'} = a^*))$ is identifiable if, in its ST-DCM, all backdoor paths from the treatment $Q_{i,t'}^A$ to the outcome $Q_{k,t}^Y$ are blocked by conditioning on the observable spatio-temporal history. This history comprises the macro-states of unit $i$ and its neighbors $\calN(i)$ prior to time $t'$.
\end{theorem}

\begin{proof}[Proof Sketch]
The logic directly parallels that of Proposition \ref{prop:id_t_hcm_adj}. The adjustment set is expanded to include the histories of neighbors, $\{\bm{S}_{j,<t'}\}_{j \in \calN(i)}$, to block confounding paths that traverse space. After this spatio-temporal adjustment, the remaining confounding is purely from the static, unit-level variables $\{U_k\}$. As before, this static confounding is handled by leveraging the subunit-level variation within each unit to estimate conditional effects, which are then averaged. The full proof is in Appendix~\ref{app:thm_id_st_hcm_adj}.
\end{proof}

\subsection{Identifiability via Dynamic Instruments}
\label{subsec:id_instrument}

When a sufficiently rich history is not observable, we can achieve identification by leveraging an exogenous source of variation---an instrumental variable.

\begin{proposition}[Identifiability in T-HCMs via Instruments]
\label{prop:id_t_hcm_iv}
In a T-HCGM, the causal effect of $A_{ij,t'}$ on $Y_{ij,t}$ is identifiable if there exists a valid subunit-level instrumental variable $Z_{ij,t'}$. A valid instrument must (1) be a cause of $A_{ij,t'}$ (relevance) and (2) have no unblocked backdoor path to $Y_{ij,t}$ in the DCM that does not pass through $A_{ij,t'}$.
\end{proposition}

\begin{proof}[Proof Sketch]
A valid instrument $Z_{ij,t'}$ in the DCM, represented by $Q_{i,t'}^Z$, is d-separated from all confounders, including the static $U_i$ and any unobserved history. This allows us to formulate a Fredholm integral equation of the first kind: $P(Y_t | Z_{t'}=z) = \int P(Y_t | \doop(A_{t'}=a)) P(A_{t'}=a | Z_{t'}=z) da$. Since the terms involving the observational distribution are identifiable, the causal effect $P(Y_t | \doop(A_{t'}=a))$ can be recovered by solving this integral equation. The full proof is in Appendix~\ref{app:prop_id_t_hcm_iv}.
\end{proof}

This result also generalizes to the spatio-temporal framework, requiring a correspondingly stronger instrument.

\begin{theorem}[Identifiability in ST-HCMs via Instruments]
\label{thm:id_st_hcm_iv}
In an ST-HCGM, the causal effect $P(Y_{k,t} | \doop(A_{i,t'} = a^*))$ is identifiable if there exists a valid spatio-temporal subunit-level instrumental variable $Z_{ij,t'}$. A valid instrument must satisfy the exclusion restriction with respect to the entire system: its node $Q_{i,t'}^Z$ in the ST-DCM must be d-separated from all static confounders $\{U_k\}$ and the complete, unobserved history of the entire spatio-temporal system.
\end{theorem}
\begin{proof}[Proof Sketch]
The proof is identical in structure to that of Proposition \ref{prop:id_t_hcm_iv}. The instrument's exogenous variation allows us to establish the same integral equation relating the observational distribution to the desired causal effect. The key difference lies in the strength of the required assumption: the instrument must be exogenous not just to the history of its own unit, but to the history of the entire interacting system. Given such an instrument, identification follows. The full proof is in Appendix~\ref{app:thm_id_st_hcm_iv}.
\end{proof}

\subsection{ATE with Time-Invariant Spatial Confounding}
To demonstrate our framework's practical power, we focus on a canonical yet challenging problem in real-world applications: \textbf{estimating average treatment effects with unobserved, time-invariant spatial confounding.} This scenario is particularly insightful because traditional spatio-temporal methods often fail, while our hierarchical structure provides a unique solution.

\subsubsection{Problem Formulation}
We consider an ST-HCM where the primary challenge is an unobserved, time-invariant, unit-level confounder, $U_i$. This confounder affects the treatment assignment, outcomes, and potentially the dynamics of neighboring units through spatial spillover, creating complex confounding across both time and space.

Our goal is to identify and estimate the \textbf{Average Treatment Effect (ATE)} of a global, static treatment policy. Specifically, we are interested in the effect on the outcome at a future time horizon $T$, resulting from setting the subunit-level treatment $A$ to a constant value $a^*$ for all units $i$ and all time steps $t \le T$. The formal estimand is:
\begin{equation}
\label{eq:ate_definition}
\text{ATE}_T = \mathbb{E}[Y_T | \text{do}(A_t=1)] - \mathbb{E}[Y_T | \text{do}(A_t=0)],
\end{equation}
where $\mathbb{E}[Y_T | \text{do}(A_t=a^*)]$ denotes the expected global average outcome at time $T$ under the specified policy.

\subsubsection{Identification Strategy}
The identification of this ATE is made possible by our theoretical framework, specifically by applying Theorem~\ref{thm:id_st_hcm_adj}. The strategy involves three conceptual steps: (1) collapsing the ST-HCM to its equivalent ST-DCM, as guaranteed by Theorem~\ref{thm:collapse_st}; (2) applying a conditional version of sequential adjustment on the ST-DCM; and (3) leveraging the hierarchical structure to average out the unobserved confounder $U_i$.

Assuming that the dynamic and spatial confounding paths are blockable by the observable history (comprising the history of unit $i$ and its neighbors $\mathcal{N}(i)$), Theorem 3 ensures that the ATE is identifiable. The identification formula can be expressed as an expectation over the distribution of the unobserved confounders $\mathbf{U} = (U_1, ..., U_N)$:
\begin{equation*}
\label{eq:identification_formula}
\mathbb{E}[Y_T | \text{do}(A_t=a^*)] = \mathbb{E}_{\mathbf{U}} \left[ \mathbb{E}_{\text{G-Comp}}[Y_T | \text{do}(A_t=a^*), \mathbf{U}] \right].
\end{equation*}
The inner expectation, $\mathbb{E}_{\text{G-Comp}}[\cdot]$, represents the conditional causal effect for units with fixed static properties $\mathbf{U}$, which can be computed via a recursive procedure analogous to Robins' G-Computation formula. The outer expectation, $\mathbb{E}_{\mathbf{U}}[\cdot]$, is then computed by averaging these conditional effects over the population of units.

\subsubsection{Estimation Algorithm}
We translate the identification strategy into a practical, two-stage estimation algorithm.

\paragraph{Stage 1: Learning Conditional Dynamics.}
For each unit $i$, we estimate a unit-specific model, $\mathcal{M}_i$, that captures the full spatio-temporal dynamics conditional on its latent static properties $U_i$. This model is trained on the subunit-level panel data centered at unit $i$, including the histories of its neighbors.
The choice of model for $\mathcal{M}_i$ is flexible. Any method capable of capturing the relevant temporal and spatial dependencies can be employed, provided it can yield an estimate of the conditional expectation of the outcome. As our core contribution lies in the causal identification framework itself, rather than a specific predictive model, we prioritize methods that are well-established and interpretable. Suitable candidates range from semi-parametric models, such as Linear Mixed-Effects Models (LMM) , to non-parametric approaches like Gradient Boosting Machines (GBM). The specific implementation of $\mathcal{M}_i$ and further details are provided in Appendix~\ref{app:ATE_Estimation}.

\paragraph{Stage 2: Conditional G-Computation via Simulation.}
With the trained models $\{\mathcal{M}_i\}_{i=1}^N$, we simulate the effect of the intervention for each unit and then aggregate the results. The procedure is detailed in Algorithm~\ref{alg:st-hcm-ate}. For each unit $i$, we recursively simulate the system's evolution under a fixed treatment policy $a^*$. At each time step $t$, the model $\mathcal{M}_i$ predicts the outcome $\hat{y}_{i,t}(a^*)$ based on the simulated history of unit $i$ and the observed history of its neighbors. The final outcome at time $T$, $\hat{\mu}_i(T; a^*)$, represents an estimate of the conditional causal effect for that unit. Finally, we average these estimates across all units to obtain the ATE.

\begin{algorithm}[tb]
   \caption{ATE Estimation in ST-HCM}
   \label{alg:st-hcm-ate}
\begin{algorithmic}[1]
   \STATE {\bfseries Input:} Spatio-temporal data $\mathcal{D}$, time horizon $T$.
   \STATE {\bfseries Output:} Estimated Average Treatment Effect $\widehat{\text{ATE}}_T$.
   
   \vspace{0.5em}
   \STATE \textbf{Stage 1: Learning Conditional Dynamics}
   \STATE Train a set of unit-specific models $\{\mathcal{M}_i\}_{i=1}^N$ from $\mathcal{D}$.
   
   \vspace{0.5em}
   \STATE \textbf{Stage 2: Simulation and Aggregation}
   \STATE Initialize result sets $\mathcal{R}_0 \leftarrow \emptyset$, $\mathcal{R}_1 \leftarrow \emptyset$.
   \FOR{$a^* \in \{0, 1\}$}
     \FOR{$i=1$ {\bfseries to} $N$}
       \STATE Initialize simulation history $\mathcal{H}_i $.
       \FOR{$t=1$ {\bfseries to} $T$}
         \STATE Get neighbors' history $\mathcal{H}_{\mathcal{N}(i), <t}$ from $\mathcal{D}$.
          \STATE Predict outcome $\hat{y}_{i,t}^{(a^*)} \leftarrow \mathcal{M}_i(\text{do}(A_t=a^*), \mathcal{H}_i, \mathcal{H}_{\mathcal{N}(i), <t})$.
          \STATE Update $\mathcal{H}_i \leftarrow \mathcal{H}_i \cup \{(a^*, \hat{y}_{i,t}^{(a^*)})\}$.
       \ENDFOR
        \STATE Add final outcome $\hat{y}_{i,T}^{(a^*)}$ to result set $\mathcal{R}_{a^*}$.
     \ENDFOR
   \ENDFOR
   
   \vspace{0.5em}
   \STATE \textbf{Final Calculation}
   \STATE $\hat{\mathbb{E}}[Y_T | \text{do}(A_t=1)] \leftarrow \text{mean}(\mathcal{R}_{1})$.
   \STATE $\hat{\mathbb{E}}[Y_T | \text{do}(A_t=0)] \leftarrow \text{mean}(\mathcal{R}_{0})$.
   \STATE $\widehat{\text{ATE}}_T \leftarrow \hat{\mathbb{E}}[Y_T | \text{do}(A_t=1)] - \hat{\mathbb{E}}[Y_T | \text{do}(A_t=0)]$.
   \STATE \bfseries{return} \mdseries{$\widehat{\text{ATE}}_T$}
\end{algorithmic}
\end{algorithm}

\section{Experiments}
We conduct experiments on both a synthetic setting and a real-world case to evaluate ST-HCMs. We design 4 research questions to guide the evaluation: 

\begin{itemize}
    \item \textbf{RQ1 (Correctness)} Can our method recover known causal effects in an ideal synthetic setting?
    \item \textbf{RQ2 (Superiority)} How does our framework perform compared to baselines that neglect either the hierarchical structure or spatio-temporal dependencies of the data?
    \item \textbf{RQ3 (Robustness)} How robust is our method to violations of its key assumptions?
    \item \textbf{RQ4 (Applicability)} How sensitive are causal estimates to underlying structural assumptions in a real-world system, and how does ST-HCM address this challenge?
\end{itemize}

We implement our estimator with both parametric Linear Mixed Models (LMM)~\citep{Faraway2006} and non-parametric Gradient Boosting Machines (GBM) models ~\citep{10.3389/fnbot.2013.00021} to serve as the unit-specific conditional mechanism $\mathcal M_i$. Full details of our experimental setup and data generation processes are provided in Appendix~\ref{app:exp}.

\subsection{Simulation Experiments}
\subsubsection{Correctness}
To verify the fundamental correctness of our proposed estimation procedure, we first assess its ability to produce unbiased and consistent estimates of the Average Treatment Effect (ATE). We generate data from a canonical spatio-temporal hierarchical process with a known true ATE. 

Figure~\ref{fig:bias_and_consis} demonstrates the statistical properties of our ST-HCM estimator. The left panel shows the distribution of ATE estimates over multiple independent trials. The mean of this distribution (4.91) is nearly identical to the true ATE (5.0), demonstrating that our estimator is \textbf{unbiased}. The right panel shows that the absolute estimation error converges towards zero as the number of subunits per unit ($m$) increases. This confirms the \textbf{consistency} of our estimator, a key property guaranteed by our Collapse Theorem ~\ref{thm:collapse_st}. These results provide strong evidence that our framework is correctly specified and theoretically sound (\textbf{answering RQ1}).

\begin{figure}[t]
\centering
\subfloat[Unbiasedness]{%
    \includegraphics[width=0.48\columnwidth]{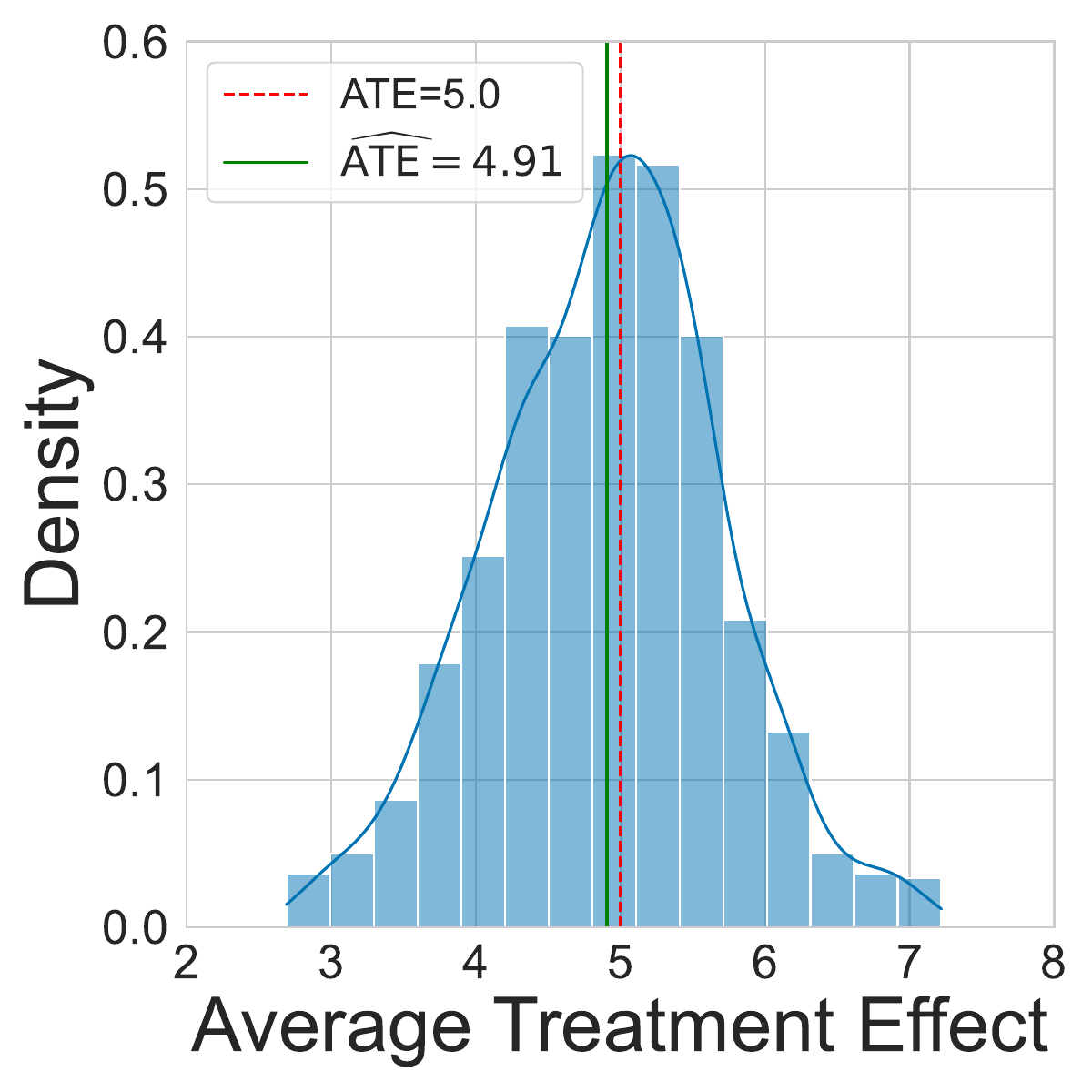}%
    \label{fig:unbiasedness}%
}
\hfill 
\subfloat[Consistency]{%
    \includegraphics[width=0.48\columnwidth]{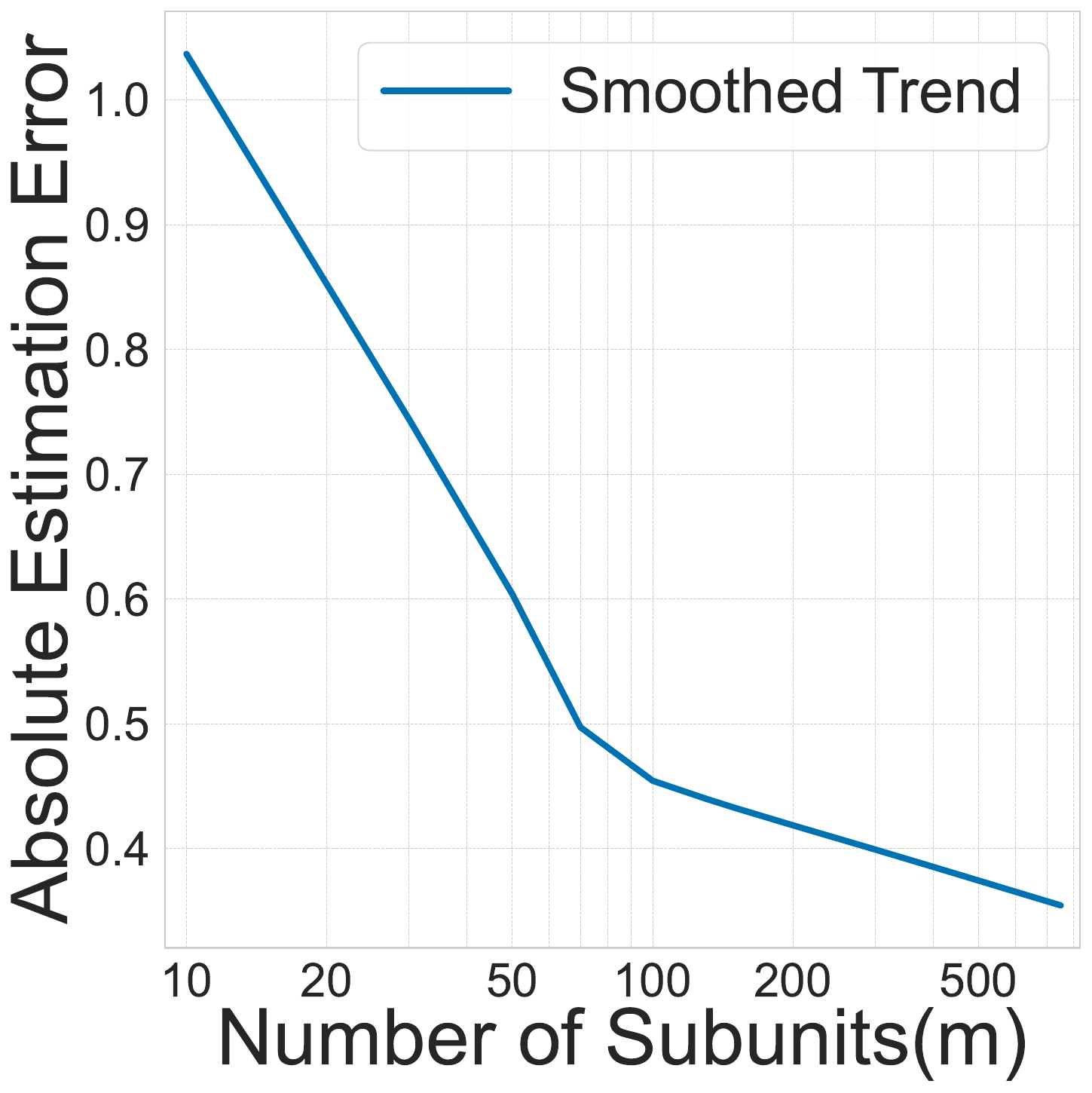}%
    \label{fig:consistency}%
}
\caption{Validation of ST-HCM estimator.}

\label{fig:bias_and_consis}
\end{figure}

\subsubsection{Superiority over Baselines}
Having established the correctness of our framework, we now conduct a rigorous stress test to evaluate its superiority over several strong baselines in the presence of unobserved confounding and spatial spillovers. To isolate the core causal identification capabilities of each method, we instantiate our framework and the hierarchical baselines using Linear Mixed Models (LMMs). This creates a challenging but well-understood linear setting where performance differences can be clearly attributed to the structural assumptions of each model.
We compare our ST-HCM (LMM) against two primary classes of baselines:
\begin{enumerate}
    \item \textbf{Aggregated Models}: A standard spatio-temporal LMM applied to data aggregated at the unit-time level, which ignores the hierarchical structure.
    \item \textbf{Non-Spatial Models}: A Temporal HCM that models the hierarchy but ignores spatial dependencies, equivalent to a panel data model with fixed effects.
\end{enumerate}

Table~\ref{tab:lmm_summary_compact}  presents the ATE absolute error for these models across a wide grid of confounding strengths ($\gamma$) and spatial spillover strengths ($\rho$). The results reveal two clear patterns.

\textbf{First, hierarchy is essential for handling confounding.} Across all levels of spatial spillover (all row groups), the error of the Aggregated model increases dramatically with the confounding strength. In contrast, both hierarchical models (T-HCM and ST-HCM) maintain low error, demonstrating their robustness to the unobserved unit-level confounder $U_i$. For instance, at $\rho=2.0$ and $\gamma=4.0$, the Aggregated model's error is 7.94, whereas the hierarchical models' errors are an order of magnitude smaller.

\textbf{Second, spatial modeling is crucial under spillover.} While the T-HCM performs well when spillover is mild ($\rho=0.0, 0.5$), its error systematically increases as $\rho$ grows. Our proposed ST-HCM, which explicitly models these spatial dependencies, consistently maintains the lowest error across nearly all settings. At $\gamma=4.0$, as $\rho$ increases from 0.0 to 2.0, the error of the T-HCM explodes from 0.11 to 2.35, while our ST-HCM's error remains stable at approximately 0.11.
These findings indicate that our ST-HCM framework, by correctly modeling both the hierarchical structure and spatio-temporal dependencies, provides substantially more accurate causal estimates than methods that ignore either dimension of the data's complexity (\textbf{answering RQ2}).

\begin{table}[t]
\centering
\begin{tabular}{@{}l@{\hspace{0.8em}}lccccc@{}}
\toprule
& & \multicolumn{5}{c}{\textbf{Confounding Strength ($\gamma$)}} \\
\cmidrule(l){3-7}
\textbf{$\rho$} & \textbf{Model} & \textbf{0.0} & \textbf{1.0} & \textbf{2.0} & \textbf{3.0} & \textbf{4.0} \\
\midrule
\multirow{3}{*}{0.0} & Aggregated & 0.173 & 1.909 & 2.886 & 3.009 & 3.086 \\
 & \textbf{T-HCM} & \textbf{0.054} & \textbf{0.070} & \textbf{0.084} & 0.126 & \textbf{0.114} \\
 & ST-HCM & 0.057 & 0.073 & 0.086 & \textbf{0.118} & 0.117 \\
\addlinespace
\multirow{3}{*}{0.5} & Aggregated & 0.164 & 1.854 & 3.038 & 3.387 & 3.785 \\
 & \textbf{T-HCM} & 0.097 & \textbf{0.068} & \textbf{0.072} & 0.128 & \textbf{0.080} \\
 & ST-HCM & \textbf{0.056} & 0.069 & 0.088 & \textbf{0.126} & 0.123 \\
\addlinespace
\multirow{3}{*}{1.0} & Aggregated & 0.189 & 2.261 & 3.624 & 4.013 & 4.260 \\
 & T-HCM & 0.132 & 0.108 & 0.161 & 0.251 & 0.224 \\
 & \textbf{ST-HCM} & \textbf{0.056} & \textbf{0.066} & \textbf{0.076} & \textbf{0.127} & \textbf{0.121} \\
\addlinespace
\multirow{3}{*}{1.5} & Aggregated & 0.211 & 3.420 & 5.051 & 5.950 & 6.734 \\
 & T-HCM & 0.174 & 0.135 & 0.335 & 0.507 & 0.606 \\
 & \textbf{ST-HCM} & \textbf{0.058} & \textbf{0.066} & \textbf{0.064} & \textbf{0.123} & \textbf{0.113} \\
\addlinespace
\multirow{3}{*}{2.0} & Aggregated & 0.228 & 3.962 & 5.620 & 6.819 & 7.942 \\
 & T-HCM & 0.329 & 0.271 & 0.894 & 1.320 & 2.354 \\
 & \textbf{ST-HCM} & \textbf{0.056} & \textbf{0.072} & \textbf{0.070} & \textbf{0.124} & \textbf{0.110} \\
\bottomrule
\end{tabular}
\caption{Mean ATE Absolute Error for different levels of confounding strength ($\gamma$) and spatial spillover ($\rho$). Both T-HCM (no spatial modeling) and ST-HCM (full model) are proposed hierarchical estimators. T-HCM is optimal when $\rho=0$, while ST-HCM excels when $\rho \geq 1$, demonstrating the importance of correctly specifying spatial dependencies.}
\label{tab:lmm_summary_compact}
\end{table}

\subsubsection{Robustness}
We assess the robustness of our ST-HCM (LMM) estimator by systematically violating two key assumptions: the time-invariance of the unobserved confounder and the causal ordering of spatial effects (Assumption~\ref{assump:spatial_order}). Our simulations are based on a challenging linear setting ($N=16, m=50, T=8$) with fixed baseline confounding ($\gamma=2.0$) and spatial spillover ($\rho=1.5$). See Appendix E for data generation details.

As illustrated in Figure~\ref{fig:time_space_ass_vio}, our proposed model demonstrates considerable robustness. When the time-invariance assumption is relaxed by introducing a temporal drift to the confounder (Figure~\ref{fig:time_invar_vio}), the error of our ST-HCM (LMM) increases mildly but remains significantly lower than that of the non-spatial T-HCM baseline. More impressively, when the spatial ordering assumption is violated by introducing simultaneous feedback loops (Figure~\ref{fig:spatial_order_vio}), the performance of our ST-HCM (LMM) remains nearly unaffected, maintaining a consistently low error and substantially outperforming both the T-HCM and Aggregated baselines across all violation intensities. This suggests our framework is a robust tool for practical applications where these assumptions may not strictly hold (\textbf{answering RQ3}).

\begin{figure}[t]
\centering
\subfloat[Time-Variant Confounding]{
    \includegraphics[width=0.46\columnwidth]{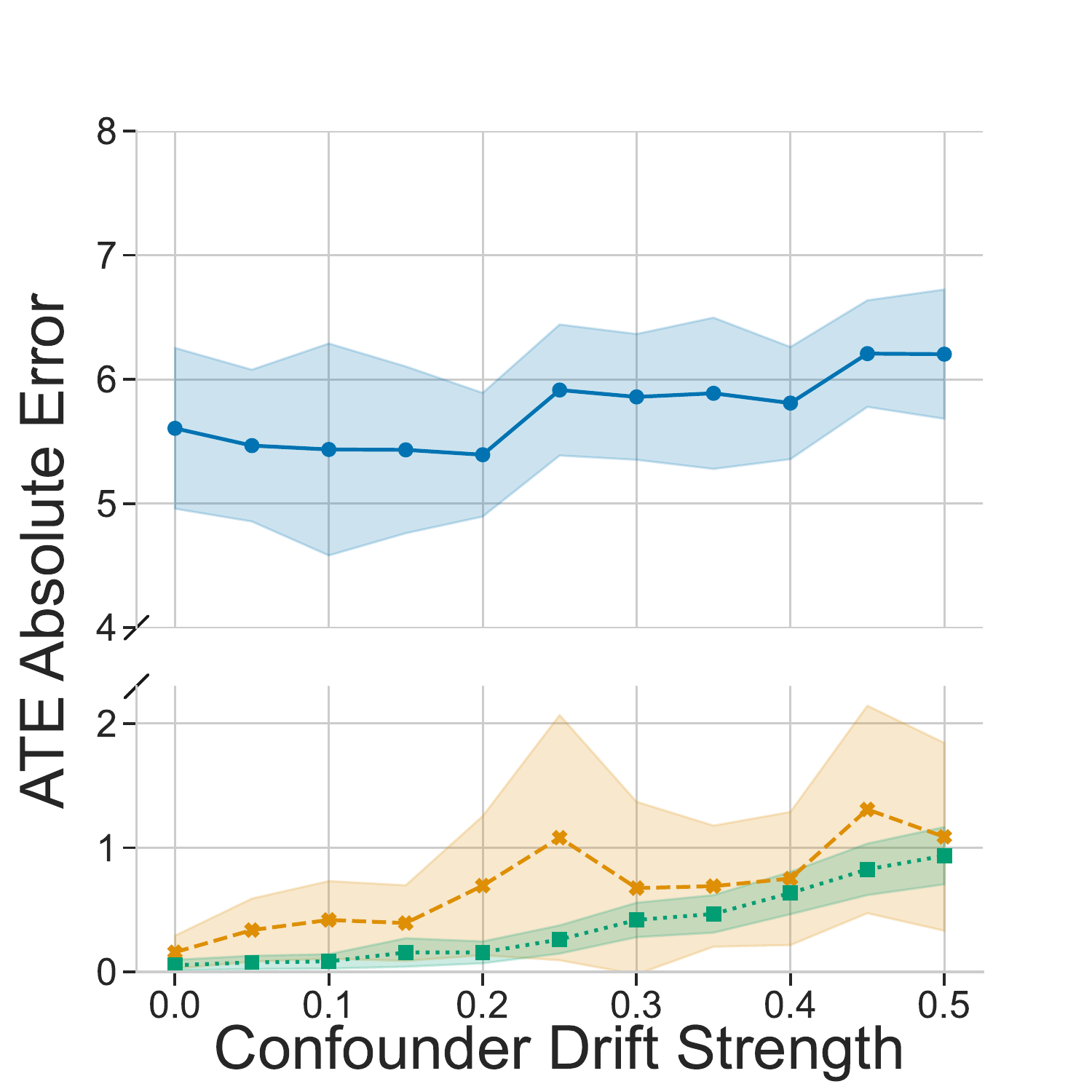}%
    \label{fig:time_invar_vio}%
}
\hfill 
\subfloat[Violation of Spatial Order]{%
    \includegraphics[width=0.46\columnwidth]{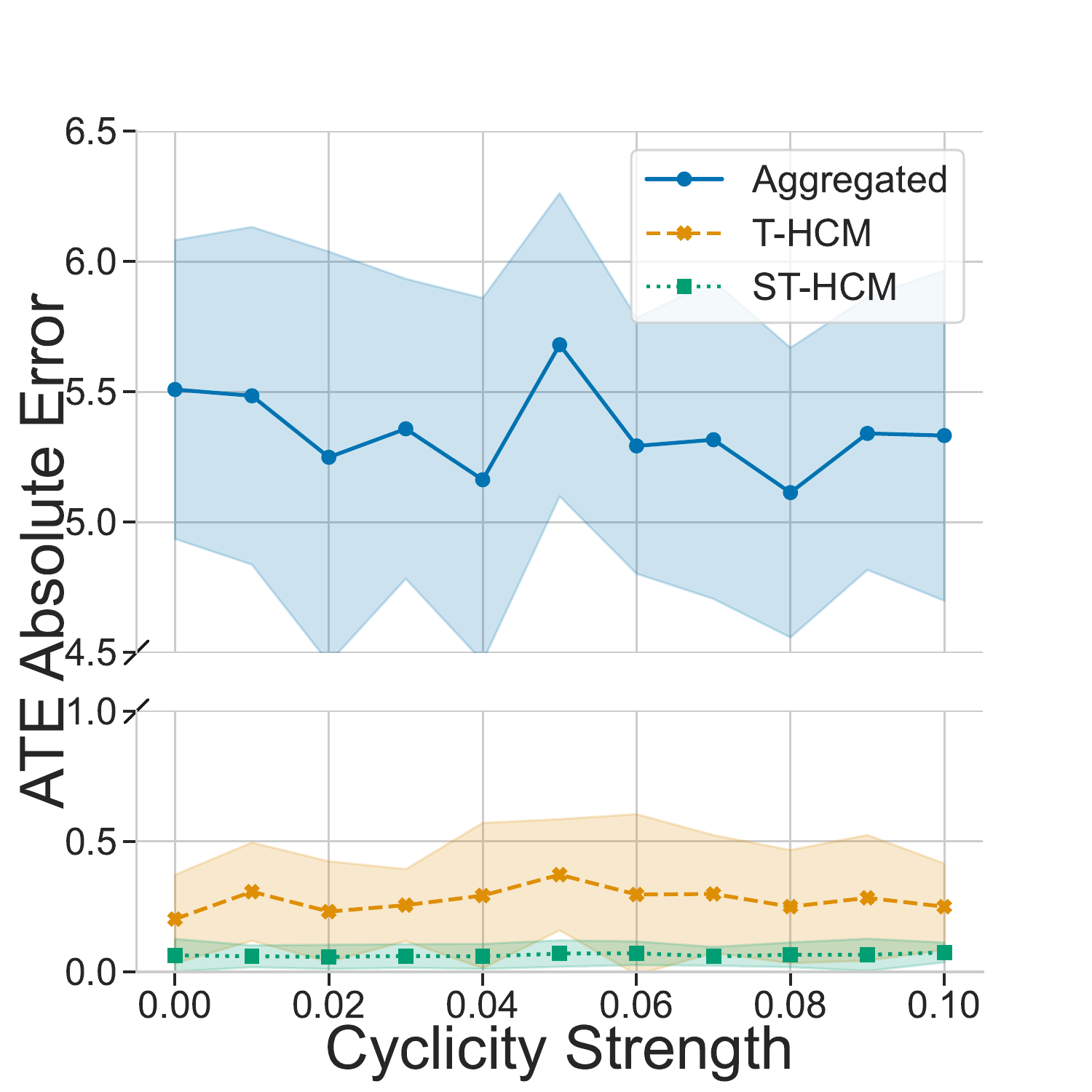}%
    \label{fig:spatial_order_vio}%
}
\caption{Robustness analysis of estimators.}
\label{fig:time_space_ass_vio}
\end{figure}

\subsection{Real-World Application: Urban Traffic}

We demonstrate the applicability of our framework in real-world using a Chicago traffic data set. Our analysis uses traffic speed data from 1,025 road segments (subunits) across 29 traffic regions (units) over a one-week period, comprising over 1.2 million observations and 1943 treatment events(Figure~\ref{fig:ate_real_world}(a)).

To capture the complex non-linearities inherent in traffic systems, we employ Gradient Boosting Machines (GBM) as the implementation for the unit-specific conditional mechanism $\mathcal{M}_i$. We compare our full ST-HCM against a T-HCM that ignores spatial effects, and an Aggregated model that ignores the hierarchy entirely.

Results summarized in Figure~\ref{fig:ate_real_world}(b) reveal that causal estimates are sensitive to structural assumptions. Progressively incorporating the hierarchical and then the spatial structure systematically attenuates the estimated ATE. Notably, this process also sharpens the posterior distribution, demonstrating that a more correctly specified model yields not only a less biased but also a more precise estimate by disentangling the treatment effect from unobserved confounders and spatial spillovers (\textbf{answering RQ4}).

\begin{figure}[t]
\centering
\subfloat[Chicago Layout]{%
    \includegraphics[width=0.3\columnwidth]{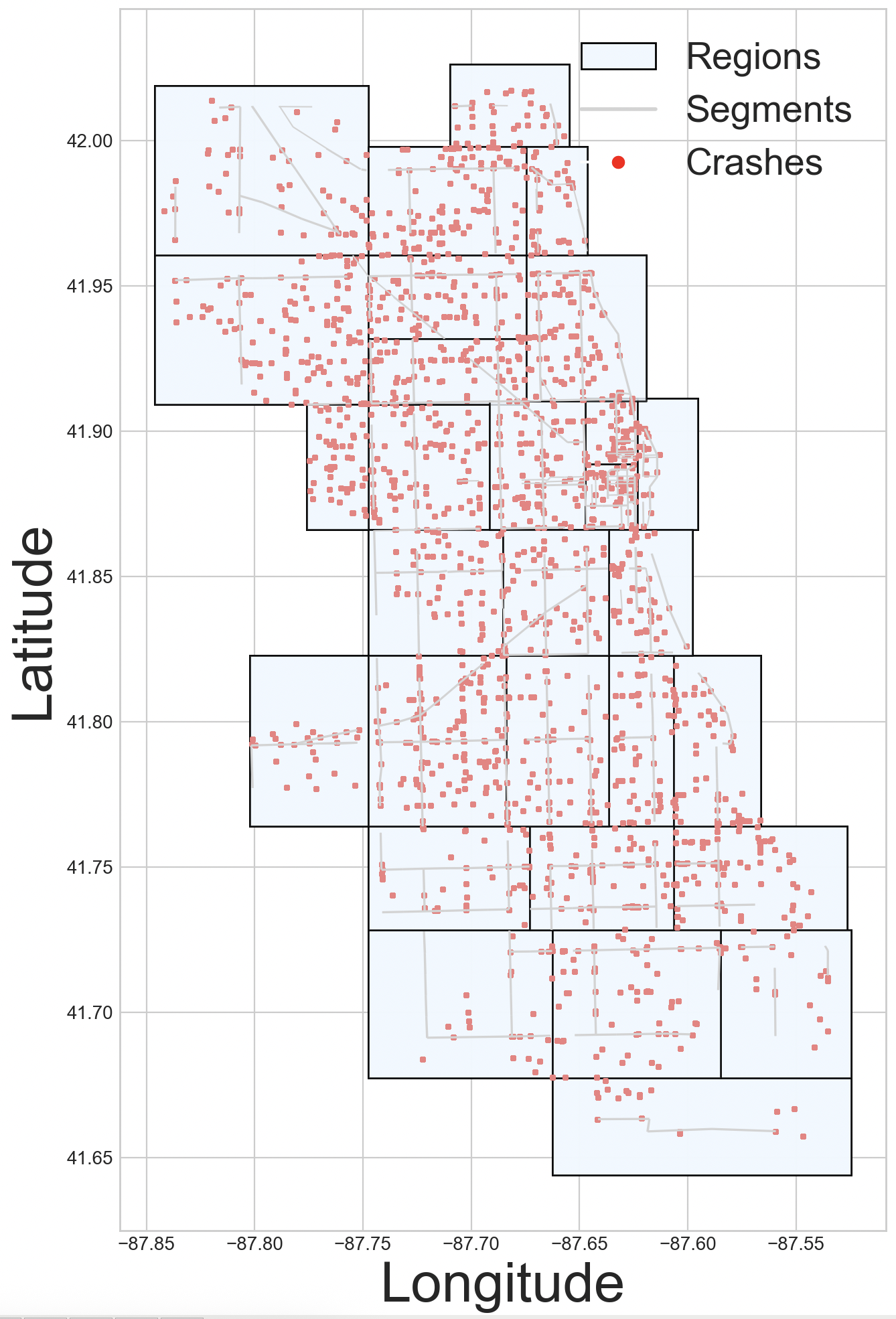}
    \label{fig:chicago}%
}
\hfill 
\subfloat[Estimated ATE Distributions]{%
    \includegraphics[width=0.54\columnwidth]{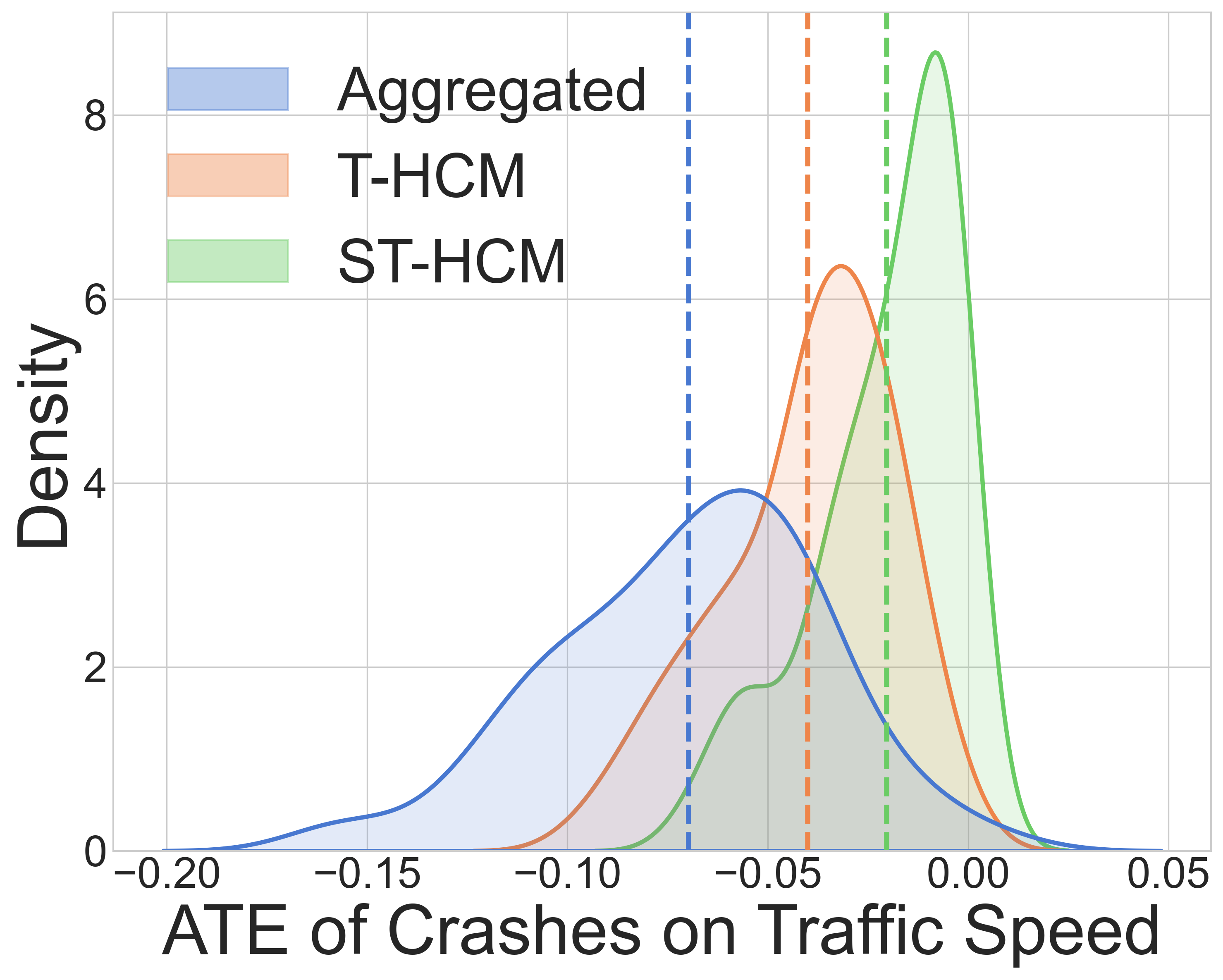}%
    \label{fig:real_world_ATE_Distri}%
}
\caption{Real-world application.}
\label{fig:ate_real_world}
\end{figure}

\section{Related Work}

\paragraph{Hierarchical causal models}
The hierarchical causal model framework ~\citep{weinstein2024hierarchicalcausalmodels}, exploits subunit variation to address unobserved unit-level confounding, generalizing fixed-effects and related econometric methods ~\citep{gelman2007data, wooldridge2010econometric}. However, they are designed for static i.i.d. settings and cannot capture temporal or spatial dependencies. Recent advances such as DML-panel methods with correlated random effects and dynamic multivariate panel models suggest possible extensions toward temporal settings ~\citep{10.1093/ectj/utaf011,tikka2024dynamiterpackagedynamic}. Our ST-HCMs extend HCMs by explicitly modeling temporal dynamics and spatial interactions while still addressing latent unit-level confounders.

\paragraph{Causal inference for spatio-temporal data}
Existing approaches fall into two categories: (1) causal discovery methods that learn dynamic causal graphs from observational data ~\citep{pamfil2020dynotearsstructurelearningtimeseries,zhao2023spatio,gong2024causal,wang2025discoveringlatentcausalgraphs} (2)  causal effect estimators that adapt frameworks like G-computation, marginal structural models, or propensity scores to dynamic or spatial settings, with recent extensions using recurrent or transformer architectures ~\citep{bica2020timeseriesdeconfounderestimating,melnychuk2022causaltransformerestimatingcounterfactual,bhattacharya2024causaleffectestimationnetwork,11017752,oprescu2025gstunetspatiotemporalcausalinference}. However, most assume no latent unit-level confounders, which is unrealistic in domains like environment or transportation. In contrast, ST-HCM address this gap by structurally incorporating unobserved, static confounders, enabling more credible causal inference.

\paragraph{Panel data and IV-based causal inference}
Panel methods such as fixed effects, difference-in-differences, and synthetic control estimate effects via within-unit variation or counterfactuals ~\citep{wooldridge2010econometric,abadie2010synthetic}, but their i.i.d. and linearity assumptions limit use in spatio-temporal contexts ~\citep{millimet2023fixed,10.1093/ectj/utae014}. IV-based approaches mitigate confounding, for example through DeepIV and semiparametric IVs ~\citep{pmlr-v70-hartford17a,10.1145/3690624.3709407,cui2025semiparametricbayesianmethodinstrumental}; and the recent Time‑dependent Instrumental Factor Model (TIFM) addresses time‑varying confounders with learned substitutes ~\citep{cheng2024instrumental}. Still, IV methods hinge on valid instruments, which are often scarce in practice. Our ST-HCM offers a flexible alternative by combining hierarchical modeling with graphical reasoning and spatio-temporal analysis.

\section{Conclusion}


In this paper, we introduced Spatio-Temporal Hierarchical Causal Models (ST-HCMs), a novel framework enabling reliable causal inference from nested spatio-temporal data. Our theoretical foundation rests on a Spatio-Temporal Collapse Theorem, guaranteeing that complex hierarchical models converge to simpler, tractable counterparts, providing rigorous causal identification. Extensive experiments verified this framework, demonstrating theoretical correctness, practical necessity over baselines ignoring hierarchy or spatial effects, and robustness to assumption violations.

\section{Acknowledgements}
We thank the reviewers for their valuable feedback.
This work was supported by Public Computing Cloud, Renmin University of China
and the fund for building world-class universities (disciplines) of Renmin University of China.


\newpage
\appendix

\section{Proof of Theorem \ref{thm:collapse_t}}  
\label{app:proof_t_hcm_collapse}

This section provides the complete proof for the convergence of a Temporal Hierarchical Causal Model (T-HCGM) to its corresponding Dynamic Collapsed Model (DCM).

\begin{proof}[Proof]
The proof proceeds by mathematical induction on the time horizon $T$.
Let $P_{col,t}$ and $P_{m,t}^{marg}$ denote the joint distributions of the macro-state history $(\bmS_1, \dots, \bmS_t)$.

\textbf{Base Case : } 
For $T=1$, we need to show that $\lim_{m\to\infty} \KL(P_{col,1} \ || \ P_{m,1}^{marg}) = 0$.
The macro-state at $t=1$ is $\bmS_1 = (X_1^U, \bmQ_1)$. Using the chain rule for KL-divergence \citep{m1991elements}, we have:
\begin{align}
    \KL(P_{col,1} \ || \ P_{m,1}^{marg}) =  \KL(P_{col}(\bmQ_1) \ || \ P_m^{marg}(\bmQ_1)) \notag \\
     + \E_{P_{col}(\bmQ_1)}[\KL(P_{col}(X_1^U|\bmQ_1) \ || \ P_m^{marg}(X_1^U|\bmQ_1))]
\end{align}

The first term, $\KL(P_{col}(\bmQ_1) \ || \ P_m^{marg}(\bmQ_1))$, is identically zero. This is because the generation of $\bmQ_1$ is, by definition, a unit-level process that does not depend on any subunit realizations, and the DCM is constructed to preserve this mechanism exactly. Thus, $P_{col}(\bmQ_1) = P_m^{marg}(\bmQ_1)$.

For the second term, we consider the inner KL-divergence, $\KL(P_{col}(X_1^U|\bmQ_1) \ || \ P_m^{marg}(X_1^U|\bmQ_1))$. The distribution $P_{col}(X_1^U|\bmQ_1)$ is determined by the true (but empty) subunit history measure, $\calP_{<1}$, while $P_m^{marg}(X_1^U|\bmQ_1)$ is determined by the empirical measure from $m$ subunits, $\hat{\calP}_{m,<1}$. By Assumption \ref{assump:mech_conv_t} (applied to an empty history), this KL-divergence converges to zero as $m \to \infty$. Since the integrand converges to zero for every $\bmQ_1$ and is non-negative, the expectation also converges to zero by the Dominated Convergence Theorem. Thus, the base case holds.

\textbf{Inductive Step:}
Inductive Hypothesis: Assume that for a time horizon $t-1$, the theorem holds:

\begin{align}
    \lim_{m\to\infty} \KL(P_{col,t-1} \ || \ P_{m,t-1}^{marg}) = 0 
\end{align}

We must prove that the theorem also holds for the time horizon $t$. We apply the chain rule for KL-divergence to the joint distribution up to time $t$:
\begin{align}
    \KL(P_{col,t} \ || \ P_{m,t}^{marg}) =  \KL(P_{col,t-1} \ || \ P_{m,t-1}^{marg})  \label{eq:term1}   \\
     + \E_{P_{col,t-1}} \left[ \KL\left( P_{col}(\bmS_t | \bmS_{<t}) \ || \ P_{m}^{marg}(\bmS_t | \bmS_{<t}) \right) \right] \label{eq:term2}
\end{align}

By the Inductive Hypothesis, Term \eqref{eq:term1} converges to zero as $m \to \infty$. The proof thus reduces to showing that Term \eqref{eq:term2}, the expected KL-divergence of the one-step transition kernels, also converges to zero. We will establish this by showing that the integrand converges to zero pointwise for any given macro-history $\bmS_{<t}$.

Let us analyze the inner KL-divergence by decomposing it for the components of $\bmS_t = (X_t^U, \bmQ_t)$:
\begin{align}
    \KL\left( P_{col}(\bmS_t | \bmS_{<t}) \ || \ P_{m}^{marg}(\bmS_t | \bmS_{<t}) \right) \notag \\
    =  \KL(P_{col}(\bmQ_t|\bmS_{<t}) \ || \ P_m^{marg}(\bmQ_t|\bmS_{<t})) \notag \\
     + \E_{P_{col}(\bmQ_t|\bmS_{<t})}[\KL(P_{col}(X_t^U|\bmQ_t, \bmS_{<t}) \notag \\
     || \ P_m^{marg}(X_t^U|\bmQ_t, \bmS_{<t}))]\label{eq:inner_kl_decomp} \notag\\
    = \text{Part A} + \text{Part B}
\end{align}

\textbf{Part A} is identically zero because the generating mechanism for the Q-variables, $p(\bm{q}_t | U_i, \bm{S}_{i,<t})$, is a unit-level process preserved exactly in the DCM, making the two conditional distributions identical.

For \textbf{Part B}, we focus on the inner KL-divergence term inside the expectation. The two conditional distributions for $X_t^U$ differ only in their dependence on the measure over the subunit history space: $P_{col}(X_t^U|\bmQ_t, \bmS_{<t})$ depends on the true measure $\calP_{<t}$, while $P_m^{marg}(X_t^U|\bmQ_t, \bmS_{<t})$ depends on the empirical measure $\hat{\calP}_{m,<t}$. The convergence of this KL term to zero follows from a continuous mapping argument. 

By Assumption A2,  the class of subunit histories is a Glivenko-Cantelli class. This guarantees the uniform convergence of the empirical measure to the true measure (a result from empirical process theory, see, e.g., \citep{Dudley_1999}), i.e., $\hat{\calP}_{m,<t} \to \calP_{<t}$ almost surely in a suitable metric as $m \to \infty$. 
By Assumption A1, the mapping from this measure to the output distribution of $X_t^U$ is continuous in the sense of KL-divergence.
The Continuous Mapping Theorem \citep{GVK164761632} states that a continuous function of a convergent sequence of random variables converges to the function of the limit. Applying this here, since the input measures converge and the mapping is continuous, the output distributions must also converge. The mode of convergence guaranteed by Assumption A1 is precisely what we need: the KL-divergence between the output distributions converges to zero.
Since the integrand in Part B converges to zero for every $\bmQ_t$, and is non-negative and appropriately bounded, the entire expectation in \textbf{Part B converges to zero} by the Dominated Convergence Theorem \citep{folland2013real}.


Since both Part A and Part B are zero in the limit, the KL-divergence of the one-step transition kernel converges to zero. Consequently, Term \eqref{eq:term2} converges to zero. This completes the inductive step.

By the principle of mathematical induction, the theorem holds for any finite time horizon $T$.

\end{proof}

\section{Proof of Theorem \ref{thm:collapse_st}}  
\label{app:proof_st_hcm_collapse}

This section provides the proof for the convergence of a Spatio-Temporal Hierarchical Causal Model (ST-HCGM) to its corresponding Spatio-Temporal Dynamic Collapsed Model (ST-DCM).

\begin{proof}
The core of the proof is to construct a ``Super-Unit" model, $\calM_{super}$, that is mathematically equivalent to the ST-HCGM but formally satisfies the definition of a T-HCGM. We will then show that the convergence of $\calM_{super}$ guaranteed by Theorem \ref{thm:collapse_t} implies the convergence stated in Theorem \ref{thm:collapse_st}.

\textbf{Step 1: Construction of the Super-Unit Model}

We define a single Super-Unit by aggregating all $N$ spatial units. The state variables and mechanisms of this Super-Unit are defined as concatenations of the corresponding elements from the individual spatial units.

\begin{itemize}
    \item \textbf{Super-Unit State Variables:}
    \begin{itemize}
        \item Static Exogenous Variable: $$U_{super} \coloneqq (U_1, U_2, \dots, U_N).$$
        \item Unit-Level Variables at time $t$: 
        $$X_{super,t}^U \coloneqq (X_{1,t}^U, X_{2,t}^U, \dots, X_{N,t}^U).$$
        \item Q-Variables at time $t$: 
        $$\bmQ_{super,t} \coloneqq (\bmQ_{1,t}, \bmQ_{2,t}, \dots, \bmQ_{N,t}).$$
        \item Macro-State at time $t$: 
        $$\bmS_{super,t} \coloneqq (X_{super,t}^U, \bmQ_{super,t}).$$
    \end{itemize}
    \item \textbf{Super-Unit Subunits:} The set of subunits for the Super-Unit is the union of all subunits from all spatial units, totaling $N \times m$. The mechanism for each unit-level variable $X_{i,t}^w$ within the Super-Unit's state depends on the empirical measure constructed from this entire collection of subunits.
    
    \item \textbf{Super-Unit Mechanism:} The mechanism $\mathbf{f}_{super}$ for the Super-Unit is a vector-valued function where the $i$-th component is the mechanism $f_i$ from the original ST-HCGM. The evolution of the Super-Unit's state is thus described by:
    \begin{align}
        \bmS_{super,t} = \mathbf{f}_{super}(\bmS_{super,<t}, U_{super}, \notag\\
        \text{Noise}_{super,t}, \{\text{Subunit Info}\}_{super})
    \end{align}
\end{itemize}

\textbf{Step 2: Verify $\calM_{super}$ is a Valid T-HCGM}

We must verify that $\calM_{super}$ adheres to the definition of a T-HCGM. The most critical requirement is that its underlying causal graph, when unrolled over time, must be a Directed Acyclic Graph (DAG). The structure of $\calM_{super}$ is inherently acyclic across time steps, as any state at time $t$ can only depend on states at times $t' \le t$. The crucial step is to establish acyclicity within a single time step $t$. The dependencies among the components of the state vector $\bmS_{super,t} = (\bmS_{1,t}, \dots, \bmS_{N,t})$ are governed by the spatial interactions of the original ST-HCGM. Assumption ~\ref{assump:spatial_order} (Spatial Causal Ordering) imposes a strict, global ordering on these contemporaneous interactions. Specifically, any component $\bmS_{i,t}$ can be influenced by another component $\bmS_{j,t}$ only if $j < i$. This enforces a directed, acyclic structure on the dependencies within any given time slice. Since the graph is acyclic both across and within time steps, the unrolled graph for $\calM_{super}$ is a DAG. The other definitional aspects of a T-HCGM, such as the presence of a static confounder and time-varying noise, are satisfied by construction. Thus, $\calM_{super}$ is a valid T-HCGM.

\textbf{Step 3: Verify $\calM_{super}$ Satisfies Premises of Theorem 1}

We now verify that $\calM_{super}$ satisfies the assumptions required by the temporal collapsing theorem (Theorem \ref{thm:collapse_t}). The state space of the Super-Unit is a finite Cartesian product of Polish spaces, which is itself a Polish space \cite{kechris2012classical}, satisfying the regularity condition. The mechanism convergence assumption for $\calM_{super}$ is also met, as it is a direct consequence of the corresponding assumption on the original ST-HCGM, combined with the uniform convergence of the aggregated empirical measure over all subunits, which is guaranteed by the Glivenko-Cantelli class assumption.

\textbf{Step 4: Invoking Theorem 1}
Since $\calM_{super}$ is a valid T-HCGM that satisfies all the necessary premises, we can directly invoke Theorem \ref{thm:collapse_t}. This theorem guarantees that the marginalized distribution of the Super-Unit model converges to its collapsed distribution in KL-divergence:
$$ \lim_{m\to\infty} \KL(P_{super,col,T} \ || \ P_{super,m,T}^{marg}) = 0 $$
where $P_{super,col,T}$ is the distribution of the collapsed Super-Unit model and $P_{super,m,T}^{marg}$ is its marginalized distribution.

The final step is to establish the equivalence between the distributions of the Super-Unit model and the original ST-HCGM. By construction, the macro-state space of the ST-HCGM is identical to the state space of the Super-Unit model. Consequently, the marginalized distribution $P_{m,T,N}^{marg}$ of the ST-HCGM is the same as $P_{super,m,T}^{marg}$. Similarly, the Spatio-Temporal Dynamic Collapsed Model (ST-DCM) describes the exact same limiting process as the collapsed model of the Super-Unit, meaning $P_{col,T,N} = P_{super,col,T}$. Given these equivalences, the convergence result for the Super-Unit model directly implies the convergence for the Spatio-Temporal model.

\end{proof}

\section{Proofs for Identifiability}
\label{app:proof_id_theorems}
This section provides proofs for the identifiability theorems for Temporal Hierarchical Causal Models (T-HCMs) and Spatio-Temporal Hierarchical Causal Models (ST-HCMs).

\subsection{Proof of Proposition \ref{prop:id_t_hcm_adj}}  
\label{app:prop_id_t_hcm_adj}

We first restate the propositions concerning identifiability in the purely temporal setting. The proofs operate on the Dynamic Collapsed Model (DCM), whose existence and convergence properties are guaranteed by the temporal collapsing theorem.


\begin{proof}
The problem of identifying the causal effect is equivalent to identifying the post-intervention distribution $P(Q_{i,t}^Y | \doop(Q_{i,t'}^A = \delta_{a^*}))$ in the DCM. The primary obstacles are the backdoor paths connecting $Q_{i,t'}^A$ and $Q_{i,t}^Y$. These paths arise from two sources: (1) dynamic confounding from the system's history, and (2) static confounding from the time-invariant latent variable $U_i$.

By the premise of the proposition, all dynamic backdoor paths are intercepted by the observable macro-history $\bm{S}_{i,<t'}$. We can therefore apply the backdoor adjustment formula \citep{10.1214/09-SS057}, integrating over the distribution of this history to block these paths:
$$ P(Q_{i,t}^Y | \doop(Q_{i,t'}^A)) = \int P(Q_{i,t}^Y | \doop(Q_{i,t'}^A), \bm{s}_{<t'}) P(\bm{s}_{<t'}) d\bm{s}_{<t'} $$
After conditioning on a specific history realization $\bm{s}_{<t'}$, the only remaining unblocked backdoor path is the one mediated by the static confounder, $Q_{i,t'}^A \leftarrow U_i \rightarrow Q_{i,t}^Y$. At this stage, we leverage the hierarchical structure. The law of total probability allows us to express the conditional interventional distribution as an expectation over $U_i$:
\begin{align}
    &P(Q_{i,t}^Y | \doop(Q_{i,t'}^A), \bm{s}_{<t'}) \notag \\
    &= \E_{U_i | \bm{s}_{<t'}} \left[ P(Q_{i,t}^Y | \doop(Q_{i,t'}^A), \bm{s}_{<t'}, U_i) \right] 
\end{align}

Crucially, once we condition on $U_i$, all backdoor paths are blocked. Thus, the $\doop$ operator can be removed, and the inner term becomes an estimable quantity from the observational distribution of a specific unit $i$ (for which $U_i$ is fixed). While $U_i$ is unobserved, we can estimate this expectation by first estimating the unit-specific effect from its subunit-level panel data, and then averaging these estimates across all units. Since all components of the final expression are computable from observational data, the causal effect is identifiable.
\end{proof}

\subsection{Proof of Proposition \ref{prop:id_t_hcm_iv}}
\label{app:prop_id_t_hcm_iv}

\begin{proof}
A valid instrumental variable $Z_{ij,t'}$, represented by the node $Q_{i,t'}^Z$ in the DCM, is by definition d-separated from all confounding sources, including the static $U_i$ and any unobserved history. This exogeneity allows us to relate the desired causal effect to the observational distribution. Let $g(a) \coloneqq P(Y_t | \doop(A_{t'}=a))$ be the target causal effect distribution. The exclusion restriction of the instrument implies that the observed association between the instrument $Z_{t'}$ and the outcome $Y_t$ is mediated entirely through the treatment $A_{t'}$. This relationship can be formalized by the following integral equation:
\begin{align}
    P(Y_t | Z_{t'}=z) = \int g(a) \cdot P(A_{t'}=a | Z_{t'}=z) da
\end{align}
This is a Fredholm integral equation of the first kind. The left-hand side, $P(Y_t | Z_{t'}=z)$, and the kernel of the integral, $P(A_{t'}=a | Z_{t'}=z)$, are both identifiable from the observational joint distribution of variables in the DCM. Under standard regularity conditions that ensure the invertibility of the integral operator defined by the kernel (a completeness condition, see \citep{newey2003instrumental}), this equation can be solved for the unknown function $g(a)$. Since $g(a)$ can be uniquely determined from observational quantities, the causal effect is identifiable.
\end{proof}

\subsection{Proof of Theorem \ref{thm:id_st_hcm_adj}}  
\label{app:thm_id_st_hcm_adj}



\begin{proof}
The proof directly generalizes that of Proposition \ref{prop:id_t_hcm_adj}. The problem is first translated to the ST-DCM via the spatio-temporal collapsing theorem. The set of confounding backdoor paths now includes those that traverse through the histories of spatial neighbors. By the premise of the theorem, this entire set of dynamic and spatial backdoor paths is blocked by conditioning on the observable spatio-temporal history, $\text{Hist}_{ST,<t'}^{obs}$. Applying the backdoor adjustment formula yields:
\begin{align}
    P(Q_{k,t}^Y | \doop(Q_{i,t'}^A)) = \int P(Q_{k,t}^Y | \doop(Q_{i,t'}^A), h_{st}) P(h_{st}) dh_{st}
\end{align}
After this adjustment, the only remaining backdoor paths are those mediated by the set of all static confounders $\{U_j\}_{j=1}^N$. As in the temporal case, the hierarchical structure allows us to overcome this. By conditioning on the full set of static confounders $U_{all} = (U_1, \dots, U_N)$, all backdoor paths become blocked. The resulting conditional interventional distribution can be estimated from the observational data within unit-groups that share the same (unobserved) values of $U_{all}$. The final estimand is obtained by averaging over the population of units. Therefore, the causal effect is identifiable.
\end{proof}


\subsection{Proof of Theorem \ref{thm:id_st_hcm_iv}}
\label{app:thm_id_st_hcm_iv}

\begin{proof}
The proof structure is identical to that of Proposition \ref{prop:id_t_hcm_iv}. A valid spatio-temporal instrument $Z_{ij,t'}$, represented by $Q_{i,t'}^Z$ in the ST-DCM, must be exogenous with respect to the entire confounding system. This means it must be d-separated from the set of all static confounders $\{U_k\}$ and the complete unobserved history of the entire spatio-temporal system.

Given such an instrument, we can again formulate the Fredholm integral equation relating the causal effect distribution $g_k(a) \coloneqq P(Y_{k,t} | \doop(A_{i,t'}=a))$ to observational quantities:
\begin{align}
    P(Y_{k,t} | Z_{i,t'}=z) = \int g_k(a) \cdot P(A_{i,t'}=a | Z_{i,t'}=z) da
\end{align}
The terms on the left and the kernel of the integral are identifiable from the observational distribution of the ST-DCM. Provided the completeness condition holds for the instrument, this equation has a unique solution for $g_k(a)$. Thus, the full causal effect distribution is identifiable.
\end{proof}

\section{Details for ATE Estimation}
\label{app:ATE_Estimation}

\subsection{Linear Mixed-Effects Model Specification}
\label{app:ATE_Estimation_lmm}

This section details the implementation of our two-stage estimation algorithm when using a semi-parametric Linear Mixed-Effects Model (LMM) as the unit-specific conditional mechanism, $\mathcal{M}_i$. LMMs, also known as panel data models with fixed effects in econometrics, are a natural choice for this problem as they are explicitly designed to handle unobserved time-invariant heterogeneity.

\paragraph{Stage 1: Learning Conditional Dynamics with LMMs}
The first stage aims to estimate the conditional expectation of the outcome, $\mathbb{E}[Y_{ij,t} | \text{History}_{i,<t}, U_i]$, under the assumption of linear relationships. Unlike the non-parametric approach where an independent model is trained for each unit, the LMM leverages a more statistically efficient \textbf{pooled regression} across all units and subunits.

The core idea is to model the static, unobserved confounder $U_i$ as a \textbf{unit-specific fixed effect}. This is achieved by including a dummy variable for each unit in the regression model. This approach effectively absorbs all time-invariant differences between units, including the effect of $U_i$, into these unit-specific intercepts.

We fit a single, global LMM to the entire subunit-level dataset. The model is specified by the following linear formula:
\begin{equation}
\begin{split}
    Y_{ij,t} = &\ \alpha_i + \beta_A A_{ij,t} + \beta_{\text{temp}} \bar{Y}_{i,t-1} \\ 
               &\ + \rho \bar{Y}_{\mathcal{N}(i), t-1} + \epsilon_{ij,t}
\end{split}
\label{eq:lmm_formula}
\end{equation}
where $\alpha_i$ represents the fixed effect for unit $i$, and the coefficients $\beta_A$, $\beta_{\text{temp}}$, and $\rho$ are shared across all units. The treatment effect, $\beta_A$, is the primary parameter of interest estimated in this stage.

\paragraph{Stage 2: Conditional G-Computation via Simulation}
With the fitted LMM, we perform a conditional G-computation procedure as in Algorithm~1 to estimate the average treatment effect. This stage remains conceptually identical to the non-parametric case, but the prediction at each step uses the fitted linear model instead of a GBM.

The simulation for each unit $i$ under a fixed policy $a^*$ proceeds recursively. At each time step $t$, we predict the potential outcome $\hat{y}_{i,t}(a^*)$ using the fitted coefficients from Equation~\ref{eq:lmm_formula}:
\begin{itemize}
    \item The fixed treatment assignment, $a^*$, is used for the term $\beta_A a^*$.
    \item The unit-specific intercept, $\hat{\alpha}_i$, learned during Stage 1, is used. This is the crucial step that conditions the prediction on the unit's unique characteristics (i.e., on $U_i$).
    \item The temporal lag is taken from the \textit{simulated} history of unit $i$.
    \item The spatial lag is taken from the \textit{observed} history of unit $i$'s neighbors.
\end{itemize}
This process is repeated until we obtain the final potential outcome $\hat{y}_{i,T}(a^*)$, which represents the estimate of $\mathbb{E}_{G\text{-Comp}}[Y_T | do(A_t = a^*), U=u_i]$.

\paragraph{Final ATE Calculation}
The final ATE is computed by averaging the simulated potential outcomes across all units, exactly as in the GBM case:
\begin{equation}
    \widehat{\text{ATE}}_T = \frac{1}{N}\sum_{i=1}^{N} \hat{y}_{i,T}(1) - \frac{1}{N}\sum_{i=1}^{N} \hat{y}_{i,T}(0).
\end{equation}

\subsection{Gradient Boosting Machines Specification}
\label{app:ATE_Estimation_GBM}

This section provides a description of the two-stage estimation algorithm presented in the main text, specifically for the non-parametric Gradient Boosting Machine (GBM) implementation of the unit-specific conditional mechanism, $\mathcal M_i$.

\paragraph{Stage 1: Learning Conditional Dynamics with GBMs}
The first stage involves learning the conditional dynamics for each unit. The goal is to estimate the conditional expectation of the outcome, $\mathbb{E}[Y_{i,t} | \text{History}_{i,<t}, U_i]$, where the conditioning on the unobserved, time-invariant confounder $U_i$ is crucial.

To achieve this, we train an \textbf{independent GBM for each unit $i$}, which we denote as $\mathcal M_i$. Training a separate model for each unit is a standard strategy to implicitly condition on unit-specific fixed effects. In this way, each model $\mathcal M_i$ learns the unique dynamic patterns of its corresponding region, effectively accounting for the influence of that region's unique time-invariant characteristics, $U_i$.

The target variable for each model $\mathcal M_i$ is the subunit-level speed, $Y_{ij,t}$. The feature set consists of the relevant temporal and spatial history, including lagged speeds, lagged treatments, and the historical states of neighboring units, as specified in our experimental setup.

\paragraph{Stage 2: Conditional G-Computation via Simulation}
With the set of trained models $\{\mathcal M_i\}_{i=1}^N$, we perform a conditional G-computation procedure as outlined in Algorithm 1 to estimate the potential outcomes under different global treatment policies. This stage involves two parallel simulations for each unit $i$, one for the policy $do(A_t = 0)$ and another for $do(A_t = 1)$.

The simulation for a single unit $i$ under a fixed policy $a^*$ proceeds recursively from $t=1$ to $T$. At each time step $t$, we predict the potential outcome $\hat{y}_{i,t}(a^*)$ using the trained model $\mathcal M_i$. The input features for this prediction are constructed from:
\begin{itemize}
    \item The fixed treatment assignment, $a^*$.
    \item The \textit{simulated} history of unit $i$ up to time $t-1$, which is built using the outcomes predicted in previous steps of the simulation.
    \item The \textit{observed} history of unit $i$'s neighbors, $\{S_{k,<t}\}_{k \in \mathcal{N}(i)}$.
\end{itemize}
This recursive process is repeated until we obtain the final potential outcome for unit $i$ at the time horizon $T$, denoted $\hat{y}_{i,T}(a^*)$. This value serves as an estimate of the inner expectation in our identification formula, $\mathbb{E}_{G\text{-Comp}}[Y_T | do(A_t = a^*), U=u_i]$.

\paragraph{Final ATE Calculation}
After completing the simulations for all $N$ units under both treatment policies ($a^*=0$ and $a^*=1$), we obtain two sets of potential outcomes: $\{\hat{y}_{i,T}(0)\}_{i=1}^N$ and $\{\hat{y}_{i,T}(1)\}_{i=1}^N$. The final ATE is computed by averaging across these unit-specific estimates, which corresponds to the outer expectation over the distribution of $U$:
\begin{equation}
    \widehat{\text{ATE}}_T = \frac{1}{N}\sum_{i=1}^{N} \hat{y}_{i,T}(1) - \frac{1}{N}\sum_{i=1}^{N} \hat{y}_{i,T}(0).
\end{equation}
This non-parametric, simulation-based approach allows us to estimate the causal effect while accommodating the complex, non-linear dynamics learned by the GBMs.

\subsection{Gaussian Process Model Specification}
\label{app:ATE_Estimation_GP}

In our case study, we can also employ a unit-specific Gaussian Process (GP) model\citep{rasmussen:williams:2006}, denoted as~$\mathcal{M}_i$, to estimate the conditional dynamics for each unit~$i$. This appendix provides the detailed specification of this GP model.

The GP framework offers a principled, non-parametric approach to learn the complex, non-linear function governing the spatio-temporal dynamics, conditional on the unit's latent static properties~$U_i$. For each unit~$i$, we model the subunit-level outcome~$Y_{ij,t}$ as being generated by a latent function~$g_i(\cdot)$ corrupted by i.i.d. Gaussian noise:
\begin{equation}
    Y_{ij,t} = g_i(\mathbf{x}_{ij,t}) + \eta_{ij,t}, \quad \eta_{ij,t} \sim \mathcal{N}(0, \sigma_n^2)
\end{equation}
where~$\mathbf{x}_{ij,t}$ is the input feature vector for subunit~$j$ in unit~$i$ at time~$t$, and~$\sigma_n^2$ is the noise variance. We place a GP prior on the latent function~$g_i$:
\begin{equation}
    g_i(\cdot) \sim \mathcal{GP}(m_i(\cdot), k_i(\cdot, \cdot'))
\end{equation}
where~$m_i(\cdot)$ is the mean function and~$k_i(\cdot, \cdot')$ is the covariance (kernel) function. The use of GPs for spatio-temporal modeling is a well-established practice in \citep{gelfand2010handbook,creswik11}

\subsubsection{Input Feature Vector}
The input feature vector~$\mathbf{x}_{ij,t}$ is designed to encode all relevant information for predicting the outcome. For our specific problem, it comprises:
\begin{itemize}
    \item \textbf{Treatment:} The current treatment assignment~$A_{ij,t}$.
    \item \textbf{Time Stamp:} The current time index~$t$.
    \item \textbf{Lagged Outcome:} The subunit's own outcome from the previous time step,~$Y_{ij,t-1}$.
    \item \textbf{Spatial Lag:} A summary of neighbors' outcomes from the previous time step, defined as the average outcome in the neighborhood~$\mathcal{N}(i)$:
    \begin{equation}
        \bar{Y}_{\mathcal{N}(i), t-1} = \frac{1}{|\mathcal{N}(i)|} \sum_{k \in \mathcal{N}(i)} \bar{Y}_{k, t-1}
    \end{equation}
    where~$\bar{Y}_{k, t-1}$ is the average subunit outcome in unit~$k$ at time~$t-1$.
\end{itemize}
The full input vector is thus:
\begin{equation}
    \mathbf{x}_{ij,t} = [A_{ij,t}, t, Y_{ij,t-1}, \bar{Y}_{\mathcal{N}(i), t-1}]^T.
\end{equation}

\subsubsection{Mean and Covariance Functions}
For simplicity and flexibility, we set the prior mean function to zero,~$m_i(\mathbf{x}) = 0$, allowing the kernel to capture the full complexity of the dynamics.

The soul of our GP model is the composite kernel function~$k_i(\mathbf{x}, \mathbf{x}')$, which we construct as a sum of kernels to model different sources of variation\citep{duvenaud_2014}. For two input vectors~$\mathbf{x}$ and~$\mathbf{x}'$, the kernel is:
\begin{equation}
    k_i(\mathbf{x}, \mathbf{x}') = k_{\text{time}}(t, t') \times k_{\text{cov}}(\mathbf{z}, \mathbf{z}')
\end{equation}
where~$\mathbf{z} = [A, Y_{\text{lag}}, \bar{Y}_{\text{spatial\_lag}}]$ represents the non-temporal features. This product structure allows for a non-stationary model where the covariance structure of the covariates can evolve over time.

\paragraph{Temporal Kernel:} We use the Mat\'ern 5/2 kernel to model temporal dependencies, which is a common choice for physical processes as it assumes the underlying function is twice-differentiable.
\begin{equation}
    k_{\text{Matern52}}(t, t') = \sigma_t^2 \left(1 + \frac{\sqrt{5}d}{l_t} + \frac{5d^2}{3l_t^2}\right) \exp\left(-\frac{\sqrt{5}d}{l_t}\right)
\end{equation}
where~$d = |t - t'|$, and~$\sigma_t^2$ and~$l_t$ are the variance and lengthscale hyperparameters, respectively.

\paragraph{Covariate Kernel:} We use the standard Radial Basis Function (RBF) kernel, also known as the squared exponential kernel, for the remaining covariates. We employ Automatic Relevance Determination (ARD), which assigns a separate lengthscale parameter~$l_d$ to each dimension~$d$ of the input~$\mathbf{z}$.
\begin{equation}
    k_{\text{RBF-ARD}}(\mathbf{z}, \mathbf{z}') = \sigma_z^2 \exp\left(-\frac{1}{2} \sum_{d=1}^{D} \frac{(z_d - z'_d)^2}{l_d^2}\right)
\end{equation}
The ARD mechanism allows the model to learn the relative importance of each feature.

\subsubsection{Model Fitting}
For each unit~$i$, the hyperparameters~$\theta_i = \{\sigma_n^2, \sigma_t^2, l_t, \sigma_z^2, \{l_d\}\}$ of the model~$\mathcal{M}_i$ are learned by maximizing the log marginal likelihood of the observed data~$\mathcal{D}_i$. This is typically done using gradient-based optimizers. The resulting trained models~$\{\mathcal{M}_i\}_{i=1}^N$ are then used in Stage 2 of our estimation algorithm (Algorithm 1) to perform counterfactual simulations.

Crucially, while $\mathcal{M}_i$ is trained as a predictive model, its role in our framework is to serve as a non-parametric estimator for the unit-specific conditional dynamics. Because $\mathcal{M}_i$ is trained on data centered at unit $i$, it implicitly conditions on the latent static properties $U_i$. Therefore, under the identification guarantees of Theorem 3, using this trained model $\mathcal{M}_i$ to simulate outcomes under a do-intervention (as performed in Stage 2 of Algorithm 1) yields a valid estimate of the conditional causal effect $E[\cdot | \text{do}(\cdot), U_i]$.

\section{Experiment Details}
\label{app:exp}

This section provides supplementary details for the experiments presented in the main paper. We elaborate on the synthetic data generation process, model implementations, and provide a description of the real-world dataset.

\subsection{Details of the simulation experiment}
\label{app:exp_simulation}

\subsubsection{Data Generation Process}
We generate synthetic data from a Spatio-Temporal Hierarchical Causal Model (ST-HCM) designed to test the core challenges of unobserved confounding and spatial spillover. The process involves $N$ units, each with $m$ subunits, observed over $T$ time steps.

A time-invariant, unobserved confounder $U_i$ is drawn for each unit $i$ from a standard normal distribution, $U_i \sim \mathcal{N}(0, 1)$. This confounder influences both treatment assignment and outcomes throughout the simulation.

The results presented in the main paper's Table 1 are based on a data generating process with linear relationships, providing a clear testbed for the causal identification capabilities of different model structures. The generation process for each subunit $(i, j)$ at time $t$ is as follows:

\paragraph{1. Treatment Assignment:} The treatment $A_{ij,t} \in \{0, 1\}$ is a binary variable drawn from a Bernoulli distribution whose probability is a non-linear function of the confounder $U_i$:
\begin{equation}
    A_{ij,t} \sim \text{Bernoulli}(\sigma(\gamma \cdot U_i - 0.5))
\end{equation}
where $\sigma(\cdot)$ is the sigmoid function and $\gamma$ is the \textbf{confounding strength}. This ensures that treatment assignment is endogenously determined by the unobserved confounder.

\paragraph{2. Outcome Generation:} The outcome $Y_{ij,t}$ is generated as a linear combination of the confounder's main effect, the treatment effect, temporal and spatial lags, and i.i.d. noise:
\begin{equation}
\begin{split}
    Y_{ij,t} = &\ \gamma \cdot U_i + \beta_A \cdot A_{ij,t} + \beta_{\text{temp}} \cdot \bar{Y}_{i,t-1} \\ 
               &\ + \rho \cdot \bar{Y}_{\mathcal{N}(i), t-1} + \epsilon_{ij,t}
\end{split}
\end{equation}
where:
\begin{itemize}
    \item $\beta_A$ is the true Average Treatment Effect (ATE).
    \item $\bar{Y}_{i,t-1}$ is the average outcome of unit $i$ at time $t-1$.
    \item $\bar{Y}_{\mathcal{N}(i), t-1}$ is the average outcome of unit $i$'s neighbors $\mathcal{N}(i)$ at time $t-1$.
    \item $\rho$ is the spatial spillover strength.
    \item $\epsilon_{ij,t} \sim \mathcal{N}(0, \sigma^2_{\epsilon})$ is the subunit-level noise.
\end{itemize}

\paragraph{Simulation Parameters:} For the experiments in Table 1, we used the following parameters: $N=16$, $m=50$, $T=8$, $\beta_A=5.0$, $\beta_{\text{temp}}=0.5$, $\sigma^2_{\epsilon}=4$. The values for $\gamma$ and $\rho$ are varied as indicated in the table.


\subsection{Robustness Analysis Data Generation}
\label{app:exp_simulation_Robustness}

This section details the data generation processes used for the robustness analysis. Our aim is to systematically introduce controlled violations of key assumptions underlying hierarchical linear models (LMMs), particularly the time-invariance of the unobserved confounder and the acyclic nature of contemporaneous spatial effects. All experiments are conducted with $N=16$ units, $m=50$ subunits, and $T=8$ time steps. The base confounding strength is $\gamma=2.0$ and spatial spillover strength is $\rho=1.5$.

\subsubsection{Violating Time-Invariance of the Confounder}
\label{app:time_invar_violation}
The standard ST-HCM (LMM) with fixed effects assumes that the unobserved unit-level confounder $U_i$ is time-invariant. To violate this assumption, we allow $U_i$ to evolve over time, transforming it into a time-varying confounder $U_{it}$. Specifically, $U_{it}$ follows a simple random walk:
\begin{align*}
U_{it} = U_{i,t-1} + \delta \cdot \epsilon_{it}^U
\end{align*}
where $U_{i,0}$ is drawn from a standard normal distribution, $\epsilon_{it}^U$ is i.i.d.~standard normal noise, and $\delta$ is the confounder drift strength parameter. A higher $\delta$ implies a more rapidly changing confounder, making it harder for fixed effects to control. The treatment assignment and outcome generation mechanisms then depend on this time-varying $U_{it}$ instead of $U_i$. When $\delta=0$, the confounder is static, corresponding to the assumption holding.

\subsubsection{Violating Spatial Ordering}
\label{app:spatial_order_violation}
Our framework assumes a causal ordering among units at each time step, implying that contemporaneous spatial influences are unidirectional (e.g., from unit $j$ to unit $i$ only if $j < i$). To violate this strict ordering and introduce simultaneous, bi-directional spatial effects, we modify the noise term of the outcome generation. Instead of being independent, the subunit-level noise $\epsilon_{ijt}$ is influenced by the average noise of its neighbors at the same time step. Specifically, the observed noise for subunit $j$ in unit $i$ at time $t$ becomes:
\begin{align*}
\tilde{\epsilon}_{ijt} = \epsilon_{ijt} + \kappa \cdot \left( \frac{1}{|\mathcal{N}(i)|} \sum_{k \in \mathcal{N}(i)} \epsilon_{kjt} \right)
\end{align*}
where $\epsilon_{ijt}$ are i.i.d.~standard normal noise, $\mathcal{N}(i)$ denotes the set of all spatial neighbors of unit $i$, and $\kappa$ is the cyclicity strength parameter. A non-zero $\kappa$ introduces contemporaneous spatial correlation in the unobserved error terms. This violation cannot be fully controlled by standard linear models, as it introduces endogenous error terms that are correlated across space within the same time step. When $\kappa=0$, the spatial ordering assumption holds, and the noise is independent across units.

\subsection{Chicago Traffic Dataset}
\label{app:exp_Chicago}

Our real-world analysis uses three publicly available datasets from the City of Chicago data portal\footnote{\url{https://data.cityofchicago.org/}}. We focus on a one-week period in January 2025 (from January 5th to January 11, 2025) for our case study.

\paragraph{Data Sources}
The core components of our dataset are:
\begin{itemize}
    \item \textbf{Regions:} The 29 official traffic regions (our \textit{units}) provide the spatial partitioning of the city. Each region is composed of two to three community areas with comparable traffic patterns.
    \item \textbf{Segments:} Arterial traffic segments (our \textit{subunits}) provide real-time speed estimates every 10 minutes. These estimates are derived from the GPS traces of Chicago Transit Authority (CTA) buses. As noted by the data provider, individual segment speeds can be highly volatile due to factors like traffic signals and intersections.
    \item \textbf{Crashes:} Records of traffic accidents provide the precise time and location of crash events, which form the basis of our treatment variable.
\end{itemize}

\paragraph{Preprocessing and Feature Engineering}
We first performed a spatial join to assign each of the 1,025 unique road segments to its corresponding traffic region(s). The crash events were then assigned to a region and subsequently assigned to the nearest road segment within the boundaries of that region.

The key step in our analysis was defining the treatment variable at the unit level. A region was considered treated at a given 10-minute time step if one or more crashes occurred within it. This \textit{unit-level treatment} allows its effect to be estimated across all subunits in the region. The final feature set for our full ST-HCM model included:
\begin{itemize}
    \item \textbf{Temporal lags:} The subunit's own speed at $t-1$ (speed\_lag1) and the unit's treatment status at $t-1$ (unit\_treatment\_lag1).
    \item \textbf{Spatial lags:} The average speed and treatment status of neighboring regions at $t-1$.
    \item \textbf{Time features:} The hour of the day and the day of the week, treated as categorical variables.
\end{itemize}

\end{document}